\documentclass[twoside]{article}

\usepackage[accepted]{aistats2023}
\usepackage{hyperref}       
\usepackage{url}            
\usepackage{booktabs}       
\usepackage{amsfonts}       
\usepackage{nicefrac}       
\usepackage{microtype}      
\usepackage[dvipsnames]{xcolor}         
\usepackage{natbib}
\usepackage{graphicx}
\usepackage{amsmath,amsthm,amssymb,amsfonts}
\usepackage{algorithm}
\usepackage{algorithmic}
\usepackage{xspace}
\usepackage{pifont}
\usepackage{multirow}
\usepackage{bbm}
\usepackage[framemethod=default]{mdframed}
\usepackage{soul}
\usepackage{tikz}
\usepackage{pgfplots}
\usepackage{caption}
\usepackage{subcaption}

\newcommand{\Dis}{{\text{Dis}}}

\usepackage[colorinlistoftodos]{todonotes}

\newcommand{\matteo}[1]{\todo[inline,color=Green!10]{\textbf{MP: }#1}}
\newcommand{\matteoout}[1]{\todo[color=Green!10]{\textbf{MP: }#1}}

\newtheorem{lemma}{Lemma}[section]

\newtheorem{theorem}{Theorem}[section]
\newtheorem{proposition}{Proposition}[section]
\newtheorem{corollary}{Corollary}[section]
\newtheorem{definition}{Definition}[section]
\newtheorem{assumption}{Assumption}[section]

\newtheorem{remark}{Remark}[section]

\def\Ebb{\mathbb{E}}

\def\*{\star}

\def\reg{\operatorname{reg}}

\def\tilreg{\widetilde{\reg}}

\def\regvar{\operatorname{regVar}}
\def\errvar{\operatorname{errVar}}
\def\wh{\widehat}

\newcommand{\E}{\Ebb}

\newcommand{\field}[1]{\mathbb{#1}}

\newcommand{\fP}{\field{P}}

\newcommand{\calS}{{\mathcal{S}}}
\newcommand{\calF}{{\mathcal{F}}}

\newcommand{\calD}{{\mathcal{D}}}
\newcommand{\calE}{{\mathcal{E}}}
\newcommand{\calH}{{\mathcal{H}}}
\newcommand{\calT}{{\mathcal{T}}}

\newcommand{\calX}{{\mathcal{X}}}

\newcommand{\calY}{{\mathcal{Y}}}

\newcommand{\calO}{{\mathcal{O}}}
\newcommand{\calW}{{\mathcal{W}}}

\newcommand{\order}{\ensuremath{\mathcal{O}}}
\newcommand{\otil}{\ensuremath{\widetilde{\mathcal{O}}}}

\newcommand{\one}{\mathbbm{1}}

\newcommand{\err}{\operatorname{err}}
\newcommand{\berr}{\overline{\err}}
\newcommand{\hinerr}{\text{WLerr}}
\newcommand{\dis}{\operatorname{DIS}}
\newcommand{\WL}{\textsc{WL}}
\newcommand{\NWL}{\textsc{NOWL}}

\newcommand{\argmin}{\operatornamewithlimits{argmin}}

\begin{document}

%

%
\runningauthor{ Chen, Sankararaman, Lazaric, Pirotta, Karamshuk, Wang, Mandyam, Wang, Fang}

\twocolumn[

\aistatstitle{Improved Adaptive Algorithm for Scalable Active Learning with Weak Labeler}

\aistatsauthor{ Yifang Chen$^*$ } 
\aistatsaddress{ Paul G. Allen School of Computer
Science \& Engineering\\
University of Washington}

\aistatsauthor{Karthik Sankararaman, Alessandro Lazaric, Matteo Pirotta, Dmytro Karamshuk, \\
\textbf{ Qifan Wang, Karishma Mandyam, Sinong Wang, Han Fang}}

\aistatsaddress{ Meta  } ]

\begin{abstract}
Active learning with strong and weak labelers considers a practical setting where we have access to both costly but accurate strong labelers and inaccurate but cheap predictions provided by weak labelers. We study this problem in the streaming setting, where decisions must be taken \emph{online}. We design a novel algorithmic template, Weak Labeler Active Cover (WL-AC), that is able to robustly leverage the lower quality weak labelers to reduce the query complexity while retaining the desired level of accuracy. 
Prior active learning algorithms with access to weak labelers learn a difference classifier which predicts where the weak labels differ from strong labelers; this requires the strong assumption of realizability of the difference classifier \citep{zhang2015active}. WL-AC bypasses this \emph{realizability} assumption and thus is applicable to many real-world scenarios such as random corrupted weak labels and high dimensional family of difference classifiers (\emph{e.g.,} deep neural nets).
Moreover, WL-AC cleverly trades off evaluating the quality with full exploitation of weak labelers, which allows to convert any active learning strategy to one that can leverage weak labelers. We provide an instantiation of this template that achieves the optimal query complexity for any given weak labeler, without knowing its accuracy a-priori. Empirically, we propose an instantiation of the WL-AC template that can be efficiently implemented for large-scale models (e.g., deep neural nets) and show its effectiveness on the corrupted-MNIST dataset by significantly reducing the number of labels while keeping the same accuracy as in passive learning.
\end{abstract}

\section{Introduction}

An active learning algorithm for classification aims to obtain an $\epsilon$-optimal hypothesis (classifier) from some given hypothesis set while requesting as few labels as possible. Under some favorable conditions, active learning algorithms can require exponentially fewer labels than passive, random sampling \citep{hanneke2014theory}. Here we consider the streaming setting where the unlabeled i.i.d.\ data comes in sequence and an active learning algorithm must make the real-time decision on whether to request the corresponding label or not. 

While traditional active learning assumes access to costly but accurate labeler,
recent works~\citep[\emph{e.g.,}][]{malago2014online,urner2012learning,mozannar2020consistent,huang2017cost} consider multiple labelers whose cost and label quality varies from each other. 
This scenario is commonly encountered in practical human-in-the-loop systems. In crowdsourcing, each sample may receive labels from more than one labeler but the reliability of each labeler is unknown --there might even exist malicious labelers~\citep{CGV-071}-- and the cost of each annotation is related to their quality. In many modern content-review applications~\citep[\emph{e.g.,}][and references therein]{Garcelon2022topk} it is  possible to automatically assign a label (possibly incorrect) by leveraging pre-trained models (e.g., trained out-of-distribution or on a different task) instead of asking to a human reviewer to provide a costly but accurate assessment of the piece \citep{lee-etal-2021-unifying}. To recover an $\epsilon$-optimal hypothesis with the minimum query cost, it is thus important to carefully trade-off the usage of a cheap but potentially imprecise labeler (\emph{weak} labeler $\calW$) and a costly but accurate labeler (\emph{strong} labeler $\calO$). We will discuss more related settings in Appendix~\ref{sec: related works (supp)}.


Authors in~\citep{zhang2015active} first gave a provable algorithm on this problem .
Their approach, however, heavily relies on the existence of a difference-classifier set  $h_\text{diff} \in \calH_\text{diff}$ to predict where the weak labeler differs from the strong labeler.  Therefore, it saves query cost by only querying the strong labeler when it disagrees with the weak one.  They assume (a) $\calH_\text{diff}$ is known to the learner in advance and (b) a nontrivial realizable low false negative rate difference-classifier in the set $\calH_\text{diff}$ exists (\emph{e.g.,} a classifier that always outputs positive labels is trivial). This difference-classifier identification procedure yields an extra $\text{VCdim}(\calH_\text{diff})$-dependent term in the query complexity (\emph{i.e.,} calls to the strong labeler) that becomes dominating when $\text{VC-dim}(\calH_\text{diff}) \gtrapprox \frac{1}{\varepsilon}$, where $\varepsilon$ is the target accuracy of output model. 

In adapting this method to real-world systems, we encounter two main limitations. First, it is not always guaranteed that a low complexity realizable $\calH_\text{diff}$ exists.  Second, their deterministic algorithm is vulnerable to random corrupted labelers~\citep[\emph{e.g.,}][]{patrini2017making,miller2014adversarial,awasthi2017power}. 
In this work, we propose an algorithm called WeakLabeler-ActiveCover (WL-AC) which removes both these assumptions
by incorporating \textit{doubly robust loss estimator} inspired by the regret minimization algorithm with hints in~\citep{https://doi.org/10.48550/arxiv.2003.01922}, and develop a corresponding semi-randomized algorithm. As a trade-off, the theoretical analysis for our algorithm requires that the conditional accuracy of weak labels has relative low variance compared to bias within each region, which is milder than the knowledge of exact $\calH_\text{diff}$ and can be usually relaxed in practice. To the best of our knowledge, this is the first time this technique has been used in active learning.

Besides the realizable difference-classifier assumption, the framework proposed by \cite{zhang2015active} also lacks flexibility to even more parsimonious active query strategy. To be specific, the traditional disagreement based active learning, including the one \citep{zhang2015active} is based on, will query \textit{every} point inside the disagreement region. This deterministic approach, although being worst-case optimal, has been shown in \citep{https://doi.org/10.48550/arxiv.1506.08669} that can be improved in some benign cases by assigning more refined query probabilities for samples inside the disagreement region. 
Therefore, to preserve such advantage in WL-AC, we need a corresponding weak label-leverage strategy. 
Here we adapt the idea of Active Cover~\citep{https://doi.org/10.48550/arxiv.1506.08669} to our setting by designing
a more adaptive weak label evaluation phase that keeps an automatic trades-off that preserve the benefits of the AC algorithm.


Theoretically, \textbf{(1)} our algorithm WL-AC retains the consistent excess risk guarantees of the AC algorithm up to a constant term, which also implies the consistent shrinking speed of the sample disagreement region as shown in Section~\ref{sec: generalization result}. To achieve this, we exploit a double robust estimator and design a refined localized weak label evaluation strategy. \textbf{(2)} Even confronting a malicious/totally misleading weak labeler, WL-AC can guarantees the same order query complexity as the aggressive result from original AC algorithm as shown in Section~\ref{sec: label complexity result each block (main)}. To achieve this, we develop a weak label evaluation strategy that automatically balances exploration to estimate the quality of a weak labeler and exploitation of the weak labeler to reduce the overall query complexity. \textbf{(3)} Our result formally characterizes when we can save query complexity w.r.t.\ the strong labelers in the non-realizable setting (Section~\ref{sec: label complexity result total (main)}).

While many works with or without theory guarantees has been wildly applied in deep learning~\cite[\emph{e.g.,}][]{gal2017deep,yoo2019learning,sener2017active,zhdanov2019diverse}, disagreement-based active learning is known to be impractical in large-scale model like neural net \citep{settles2009active}.
Here we show that the proposed algorithmic framework not only has theoretical guarantees when using AC as the base query strategy but it is also highly scalable in real-world settings.
We design a practical version that takes any heuristic streaming-based sample query methods and leverage it with weak labels and demonstrate its effect  on a standard computer vision dataset (\emph{i.e.,} corrupted MNIST~\citep{Mu2019mnistc}).
%
%
Specifically,
\textbf{(1)} for the weak label generation methods, we first test on the synthetic data to demonstrate the robustness of our algorithm against different levels and distributions of noise.
Then we use the label generated from some classifier trained on non-target domain to further show  its practical usage; 
\textbf{(2)} For the heuristic active learning our algorithm is build-upon, we test the uniform sampling as baseline and a preliminary entropy-based uncertainty sampling as a classical active query strategy.

\section{Preliminaries}
\label{sec: preliminary}

  A hypothesis class $\calH$ is given to the learner such that for each $h \in \calH$ we have $h : \calX \rightarrow \calY$.  
  Samples $(x, y) \in \calX \times \calY$ are drawn independently from an underlying distribution $\calD_*$. 
  In deriving our theoretical contributions, we assume binary classification for convenience, i.e., $\calY = \{0,1\}$, but our algorithmic framework can be easily extend to the multi-class setting as shown in the experiments. We denote the expected risk of a classifier $h \in \cal{H}$ under the distribution $\calD_*$ as $\err_{\calD_*}(h) = \E_{x,y\sim \calD_*} \left[ \one[ h(x) \neq y]\right]$ and the corresponding best hypothesis as $h^* = \argmin_h\{\err_{\calD_*}(h)\}$. 
 

\paragraph{Weak labeling oracle} 
Classical active learning assumes access to a strong labeling oracle $\calO$ that always outputs the \textit{true} label with a nontrivial cost. In this paper, we additionally assume the access to a weak labeling oracle $\calW$, from which we can query labels, denoted as $y^\WL$ with  a negligible cost. 


\begin{assumption}
\label{ass: hint oracle}
Suppose there exists an hint oracle that, given any  context domain $D \in \calX$, outputs
\begin{itemize}
    \item The corresponding weak label $y^\WL \in \calY$, $y^\WL \sim \fP_\WL(\cdot|x)$ for any $x \in \calX$. Note that the quality of the weak labels compared to the true label is unknown, so it can be arbitrarily wrong.
    \item The maximum ratio of the conditional hint error across $D$, defined as $\kappa(D) \geq \max_{x,x' \in D}{\frac{ \E_{y,y^\WL} \left[\one[y^\WL \neq y |x]\right]}{ \E_{y,y^\WL}\left[ \one[y^\WL \neq y |x']\right]}}$.  It is easy to see that $\kappa(D)$ is decreasing when $D$ shrinks.
\end{itemize}
\end{assumption}
Correspondingly, we define 
\begin{align*}
    \hinerr(D) = \max_{x \in D} \E_{y,y^\WL} \left[\one[y^\WL \neq y |x]\right].
\end{align*}

Assumption~\ref{ass: hint oracle} implies that we want the conditional errors of WL to have low variance locally within some informative region $D$. 
Moreover, from practical perspective, we show in our experiment that even some underestimation of this $\kappa(D)$ can still yield improved results compared to the algorithm without WL.

\paragraph{Disagreement coefficient} For some hypothesis class $\calH$ and subset $V \subset \calH$, the region of disagreement is defined as
$
\Dis(V)=\left\{x \in \calX: \exists h, h' \in V \text { s.t. } h(x) \neq h'(x)\right\},
$
which is the set of unlabeled examples $x$ for which there are hypotheses in $V$ that disagree on how
to label $x$. Correspondingly, the disagreement coefficient of $h^* \in \calH$ with respect to a hypothesis class $\calH$ and distribution $\nu*$ is defined as
\begin{align*}
    &\theta(r_0)=\sup_{r \geq r_0} \frac{\mathbb{P}_{x \sim \nu_*}(X \in \Dis(B(h^*, r)))}{r},
    \;\theta^* = \theta(0),
\end{align*}
and $B(h^*, r) = \{h \in \calH | \fP_x(h(x) \neq h^*(x)) \leq r\}$.

\paragraph{Protocol and goal of the problem}
We consider the streaming-based active learning problem. At each time $t \in \{1,\dots,n\}$, nature draws $(x_t,y_t)$ from $\calD_*$. The learner observes just $x_t$ and chooses whether to request $y_t$ from a strong learner incurring a cost of $c_t = 1$, or to request $y^\WL_t$ from a weak learner incurring a cost of $c_t = 0$.

The main goal of the learner is to identify a near-optimal classifier $h_\text{out}$ of desired accuracy with high probability while using as lower query cost as possible under $n$ total unlabeled samples.

That is, we want to minimize both
\begin{align}
\label{eqn: goal}
    \err(h_\text{out}) - \err(h^*)
    \quad \text{and} \quad 
    \sum_{t=1}^n \one[y_t \text{ is queried}]
\end{align}
In the rest of paper, we simply use ``label query complexity'' to refer queries to strong labeler.

\paragraph{Comparison of our results with \citep{https://doi.org/10.48550/arxiv.1506.08669}}
\label{para: way to represetn results}

While the original AC is minimax optimal w.r.t.\ measurement in equation~\ref{eqn: goal}, it is also adaptive to more benign problems.
Therefore, to illustrate that our algorithm preserves the advantages of AC, we report our results relatively to the query complexity and the unlabeled sample complexity of \NWL-AC, a variant of AC adapted to our setting.
%
We postpone the actual construction into Section~\ref{sec: result(main)}.

\paragraph{Other notations} 
For convenience, we use WL to denote weak label or weak labeler. Also when we use $\E[\cdot]$ without specified subscription, we usually refer to the expectation w.r.t. to all random variables with underlying distribution inside the bracket.

\section{WL-AC: A Template for AL with Weak Labelers}
\label{sec: algo (main)}

In this section, we show a sketch of our main algorithmic template WL-AC and postpone details (e.g., subroutines OP and WL-EVAL) to Appendix~\ref{sec: algo_theory (supp)} (for the theoretical version) and Appendix~\ref{ sec: experiment-practice (supp)} (for the practical version).

\begin{algorithm}[h!]

\small
\caption{WL-AC (template sketch)}
\label{algo: main}
\begin{algorithmic}[1]
\STATE \textbf{General inputs.} Scheduled training data collection length $L_1,L_2,\ldots$ satisfying $L_{m+1} \leq \sum_{j=1}^m L_j$, confidence level $\delta$, a candidate hypothesis set $\calH$
\STATE \textbf{A plug-able base AL strategy.} An active sampling oracle $\calO_\text{baseAL}$ defined in Def.\ref{def: baseAL oracle}
\STATE \textbf{Initialization.} epoch $m=0, \tilde{Z}_{0}:=\emptyset, \Delta_{0}:=c_{1} \sqrt{\epsilon_{1}}+c_{2} \epsilon_{1} \log 3$, where
$$
\epsilon_{m}:=\frac{32\left(\log (|\mathcal{\calH}| / \delta)+\log(\sum_{j=1}^m L_j)\right)}{\sum_{j=1}^m L_j}
$$
\FOR{ $m = 1,2,\ldots, M$}
    \STATE \textcolor{orange}{Phase 1: WL evaluation and query probability assignment}
    \STATE Based on current $D_m$ and $\err(f,\tilde{Z}_m)$, call sub-algo \WL-EVAL which will automatically stop before observing $\order(L_m)$ unlabeled samples and outputs:\\
    \label{line: wl-eval}
    \text{(1)} a boolean variable \textsc{USE-WL}, which decides whether we should switch to \textsc{Use-WL} mode \\
    (2) \WL-leveraged query probability $\Dot{P}_{m}(x)$ for all $x \in D_m$ by pessimistically estimating the \WL~quality and \\
    -- \textsc{OPTION 1 (theoretical): } solving a well-specified OP and getting a optimal result\\
    -- \textsc{OPTION 2 (practical): } finding a easy feasible solution of OP
    \STATE \textcolor{orange}{Phase 2: Training data collection based on planned query probability}
    \STATE Observed $L_m$  number of i.i.d unlabeled samples, query sample $x \in D_m$ with planned probability $\Dot{P}_m(x)$.
    \STATE For each $x_t$, add the following defined $z_t$ into collection
    \begin{align*}
        z_t = 
        \begin{cases}
            \left\{x_t, y_t, y^{\WL}_t,1 / \Dot{P}_{m}\left(x_t \right)\right\}, & \text{is queried} \\
            \left\{x_t, h_m(x_t),h_m(x_t), 1\right\}, & x_t \notin D_m\\
            \left\{x_t, 1,y_t^\WL,0\right\}, & \text{otherwise }\\
        \end{cases}
    \end{align*}
    \STATE \textcolor{orange}{Phase 3: Model and disagreement region updates using collected train data}
    \STATE Calculate the estimated error for all $h \in \calH$ using shifted doubly robust estimator
        \begin{align*}
            \err(h,\tilde{Z}_m) =
            \begin{cases}
            \sum_{\tilde{Z}_m} \ell_\text{shifted}(h(x),y,y^\WL,w) \quad &\textsc{USE-WL}\\
            \sum_{\tilde{Z}_m} \one[h(x) \neq y]w &\NWL
            \end{cases}
        \end{align*}
    \STATE Calculate the empirical best hypothesis $h_{m+1}: = \argmin_{h \in \calH} \err(f,\tilde{Z}_m)$, where $\tilde{Z}_m$ is the cumulative collected samples in Phase 2.
    \STATE  Update the disagreement region $D_{m+1}$.  \\
    -- \textsc{OPTION 1 (theoretical): } Update the active hypothesis set $A_{m+1}$ based on $\err(h,\tilde{Z}_m)$ and set $D_{m+1} = \dis(A_{m+1})$\\
    -- \textsc{OPTION 2 (practical): } Update $D_m$ by calling $\calO_\text{baseAL}$ with $h_m$ info
\ENDFOR
\RETURN $h_M$
\end{algorithmic}
\end{algorithm}

Our WL-AC template is based on the Active Cover (AC) algorithm~\citep{https://doi.org/10.48550/arxiv.1506.08669}, a disagreement-based AL algorithm for strong labelers where the learner iteratively eliminates the sub-optimal hypotheses and, furthermore, compute a refined sampling distribution inside the estimated disagreement region based on collected samples instead of sampling inside the disagreement region deterministically. Similar to AC, WL-AC works in blocks $m = \{1,\ldots, M\}$, but in a more flexible style. 
For each block $m$, WL-AC maintains an estimate of the disagreement region $D_m$ and computes a WL-leveraged sampling distribution inside $D_m$.


More precisely, in each block $m$, we have three phases. In phase 1, WL-AC collects samples to evaluate the accuracy of the weak labelers. Based on these samples, it decides whether to use WL or not and computes a query probability function $\Dot{P}_{m} : \calS \to [0,1]$ that denotes the probability of sampling a label from the strong labeler for a sample in the disagreement region $D_m$. In phase 2, WL-AC leverages $\Dot{P}_m$ for the active label collection and samples are added to $\tilde{Z}_m$. Finally, in phase 3, WL-AC computes the best empirical hypothesis using a doubly robust estimator and updates the disagreement region and active hypothesis set. Let $L_m$ be the length of block $m$, i.e., observed samples $x_t$, any block strategy such that $L_{m}\leq \sum_{i<m} L_i$ works with WL-AC.

WL-AC allows to use different solutions for computing the sampling probability $\Dot{P}_m$ and for the update of the disagreement region $D_m$. For example, the theoretical instantiation of WL-AC computes $\Dot{P}_m$ by solving a complicated optimization problem and $D_m$ over the active set of hypotheses. While the theoretical instantiation is still implementable for simple models, this algorithm is impractical for modern large-scale applications such as deep neural nets. In this case, we can leverage the flexibility of our framework to overcome this limitation. For example, WL-AC can leverage any active learning oracle $\calO_\text{baseAL}$ to approximately compute both the sampling distribution and the disagreement region (see Section~\ref{sec: algo_practice (main)} for more details).

We now explain the key ideas of WL-AC in more details.

\paragraph{\ding{182} Shifted doubly robust loss estimator in Phase 3 and the corresponding query strategy in Phase 1.}
To be robust against random corruptions while leveraging the ``good'' weak labels, inspired by \cite{https://doi.org/10.48550/arxiv.2003.01922}, we introduce the loss estimator
\begin{align*}
    &\ell_\text{shifted}(h(x),y,y^{\WL},w) \\
    &= \left(\one[h(x) \neq y] - \one[h(x) \neq y^\WL]\right) w + \one[h(x) \neq y^\WL]
\end{align*}
where $w$ is a positive weight. It is easy to see that, this estimator is unbiased and has low variance when the weak label is correct since $\one[y^\WL \neq y] \leq \one[y^\WL \neq y]w$.
By adopting such estimator, we can upper bound the performance difference between $(h,h')$ in the active set $A_m$
as 
$
    \E\left[\frac{\one\left(h(x) \neq h'(x) \wedge x \in D_{m} \wedge y^\WL \neq y\right)}{\Dot{P}_m(X)}\right],
$
instead of 
$
    \E\left[\frac{\one\left(h(x) \neq h'(x) \wedge x \in D_{m}\right)}{\Dot{P}_m(x)}\right].
$
This suggests that, to achieve the same accuracy as the original AL strategy without access to weak labels, $\Dot{P}_m(x)$ can be reduced to at most $\E[y^\WL \neq y |x] P_m(x)$, where $P_m$ is the query probability function. 

Nevertheless, without further modification, the query complexity would scale with $\berr_m(h^*)$, which is the biased estimate of excess risk of $h^*$ (as it happens in AC). This is undesired in our setting where the variance of  $\berr_m(h^*)$ may be larger than in the standard AL setting without WL.
Indeed in the case that most weak labels are wrong and the excess risk of $h^*$ is low, we have
\begin{align*}
     |(\one[h^*(x) \neq y]& - \one[h^*(x) \neq y^\WL])w|\\
     &= \one[y^\WL \neq y]w 
     \gg \one[h^*(x) \neq y] w
\end{align*}

\paragraph{\ding{183} Adaptive/localized evaluation of the weak label performance in Phase 1.}
To address this problem, we evaluate the weak labelers at the beginning of each block. This deviates from the approach in~\cite{https://doi.org/10.48550/arxiv.2003.01922} where a one-time pure exploration phase is used to evaluate the quality of the hint providers before starting to exploit them.
This addresses two issues: i) we are adaptive to the changes in the disagreement region $D_m$; ii) we set a transition to \NWL~rule by adapting to the generalized error of best hypothesis $\berr_m(h^*)$, which is estimated via $\err(h_m, \tilde{Z}_{m-1})$ ( $\berr_m(h^*)$ will be formally defined in the next section).
Similarly to~\cite{zhang2015active}, being adaptive to $D_m$ means that, instead of caring about the overall quality of the weak labeler, we only care about the quality for those samples within $D_m$ (or in the other words, close to the decision boundary).

\paragraph{\ding{184} Trade-off between weak label evaluation in Phase 1 and training data collection in Phase 2.}

The previous two techniques guarantee that the excess risk of WL-AC is consistent with the one of AC when the weak labeler quality is known. In practice, this quality is unknown and a pessimistic estimator of $\E[\one[y^\WL \neq y \wedge x \in D_m]]$ is used. However, building this estimator is not trivial. Consider the extreme case where $\E[y^\WL \neq y] \to 0$, then in order to get a pessimistic estimation close to $0$ to fully adopt the advantage of the perfect weak labels, an infinite number of samples in phase 1 is needed, which diminish the label saving efforts in phase 2. WL-AC (see full version in Alg.~\ref{algo: hint-eval} in appendix) uses a smart trades-off strategy that automatically balance query complexities between the weak label evaluation and training data collection phases.



\subsection{Practical Extension}
\label{sec: algo_practice (main)}

The original AC our framework is based on is not scalable for large-scale models (\emph{e.g.,} when $\calH$ is a deep neural network). First, although~\citet{dasgupta2007general} observed that $\one[x \in D_m]$ can be efficiently determined using a single call to the ERM oracle, in practice, even this ``oracle-efficient'' definition is undesired because it requires to the model to be retrained too often. 
Second, it is hard to explicitly calculate $D_m$ as well as estimate the sample distribution (\emph{i.e.,} explicitly solving optimization problem used in OP, see eq.~\ref{algo: op}).

While our theoretical instantiation inherits the same limitations of AC, we show how to design a more practical version overcoming these issues.
First, we propose to use any active learning strategy $\calO_\text{baseAL}$ to estimate the active sample region $D_m$. For example, the simplest oracle can do uniform sampling under a given budget. Another classical oracle is an entropy-based uncertainty sampling strategy which calculates the cross-entropy for any given $x$ and select those with high entropy. 
\begin{definition}
\label{def: baseAL oracle}
$\calO_\text{baseAL}$ takes in (1) a target context set $\calX'$, which theoretically is the full current active context set $D_m$, but can also practically be the observed unlabeled context set in the next block in the batched sample setting $\{x_t\}_{t = \tau_m +1}^{\tau_m + 2L_m}$; (2) the current trained model $h_{m}$ to estimate the informativeness of each sample; (3) and other relevant parameters such as decision threshold. Then $\calO_\text{baseAL}$ outputs a subset of $\calX'' \subset \calX'$. 
\end{definition}

Second, instead of solving the complicate OP problem for $\Dot{P}_m$, we only find a feasible solution, which is usually a uniform distribution over samples within $D_m$.
Finally, our algorithm requires the knowledge of $\kappa_m$ (see Assumption~\ref{ass: hint oracle}). We choose $\kappa(D)=1$, which is shown to be effective in our experiments. Besides these key relaxations, there are other subtleties one needs to take care such as using pseudo loss and a small validation set when updating the model $h_m$. We defer all the details in Appendix~\ref{ sec: experiment-practice (supp)} for the exact implemented algorithm.


\section{Theoretical guarantees}
\label{sec: result_main}
In this section we show that the theoretical instantiation of WL-AC preserves the advantages of AC while leveraging the informative weak labeler to improve the query complexity. Refer to Appendix~\ref{sec: algo_theory (supp)} for a detailed description of the algorithm.

In our analysis, we leverage an instantiation of WL-AC, called $\text{NOWL-AC}(\{L_m\}_{m=1}^M,\delta)$, that in each block $m$
%
%
\begin{itemize}
    \item waits without querying during $[\tau_{m-1}+1,\Dot{\tau}_m]$ and then calculates $P_m$ in the \NWL~mode as the original AC;
    \item queries labels according to the planed $P_m$ during $[\Dot{\tau}_m+1,\tau_{m}]$ and then updates the training dataset;
    \item calculates the empirical loss and updates the model in \NWL~mode.
\end{itemize}
Overall, this can be regarded as a weak version of original AC with extra unlabeled samples during $\tau_{m-1}+1$ to $\Dot{\tau}_m$. Fortunately, it is easy to show that for this version we have $\Dot{\tau}_m - \tau_{m-1} \leq \order(\tau_{m} - \Dot{\tau}_m)$ and thus, the total unlabeled sample complexity is still of the same order.

In addition, we define $N_m$ as the expected query number within block $m$ if we run \NWL-AC with the same inputs, 
\begin{align*}
    N_m = \E_m[L_m \E_{x \sim \calD_*}[P_m(x)]].
\end{align*}
where $\E_m$ is the expectation condition on history before block $m$. Later we will use this $N_m$ as a base term in label complexity analysis.




\subsection{Generalization guarantees}
\label{sec: generalization result}

Our results depend on a generalization error defined as 

$
    \berr_m(h)=
    \frac{1}{\sum_{j=1}^m L_j} \sum_{j=1}^m L_j \mathbb{E}_{\calD_*}\left[\one\left(h(x) \neq y \wedge x \in D_j\right)\right]
$

Note that the results in the AL literature are often in terms of $\err_{\calD_*}(h)$, which is an upper bound of $\berr_m(h)$. Therefore, using the generalization error helps to characterize a tighter bound. Moreover, $\berr_m(h)$ seems to be algorithm-and-empirical-distribution dependent due to the $D_j$ term but we argue that the same issues occur in the original AC analysis and can be easily reduced to some meaningful cases as shown in their paper.

\begin{theorem}
\label{them: main}
Pick any $0<\delta<1 /$ e such that $|\mathcal{H}| / \delta>\sqrt{192}$. By running the theoretical version (Algorithm~\ref{algo: main-inefficient (supp)} in appendix), we have for all epochs $m=1,2, \ldots, M$, with probability at least $1-\delta$
\begin{align*}
    \err_{\calD_*}(h) - \err_{\calD_*}(h^*)\leq \order( \Delta_m^*) \quad \text { for all } h \in A_{m+1}
\end{align*}
where $\Delta_0^* =\Delta_0 \text { and } \Delta_m^*:=c_1 \sqrt{\epsilon_m \overline{\operatorname{err}}_m\left(h^*\right)}+c_2 \epsilon_m \log \tau_m ~\text{ for } m \geq 1 $.  And this directly gives the final output guarantee 
\begin{align*}
    & \err_{\calD_*}(h) - \err_{\calD_*}(h^*) \\
    &\leq \order\biggl(\sqrt{ \frac{\log\left(|\calH| n/ \delta \right)\overline{\operatorname{err}}_M\left(h^*\right)}{n}}
    + \frac{\log\left(|\calH| n^2/\delta \right)}{n}\biggr)
\end{align*}
\end{theorem}
This generalization guarantees is consistent with the one in \NWL-AC with same inputs for every block, and therefore optimal. We want to emphasize that, to get such consistency with added WL leverage strategy, some significant modification in analysis is required and has been postponed into Appendix~\ref{sec: generalization result (supp)}.

\subsection{Label complexity analysis within each block}
\label{sec: label complexity result each block (main)}

Before giving label complexity upper bounds, we define a useful parameter $\phi_m : = \frac{\E[\one[x \in D_m]]}{\berr_m(h^*) + \log(\sum_{j=1}^m L_j) \epsilon_m}$  which intuitively describes how fast the error of best hypothesis decreases with the shrinking disagreement region, with some additional regularized term. In the following theorem, we show its connection to some standard AL parameters.

\begin{theorem}[Upper bound of $\phi_m$]
For any block $m$, we always have $ \phi_m \leq \min\{\theta^*, \frac{1}{\berr_m(h^*)+ \log(\sum_{j=1}^m)) \epsilon_m} \}$. 
\end{theorem}

In in the rest of the section, we will directly use $\phi_m$ because we believe this term well characterized the case where weak labelers help.

\begin{theorem}[Expected label complexity of block $m$]\label{thm:complex.block.m}
Let    
\begin{align*}
    &\calT_{1,m} = N_m \kappa(D_m) \hinerr_m\\
    &\calT_{2,m} = \sqrt{\ln(2M\log(n)/\delta)\kappa(D_m) N_m} \\
        &\qquad + \sqrt{\frac{L_m^2}{\sum_{j=1}^{m-1}L_j}
        \log^2\left(\frac{|\calH|n}{\delta}\right) 
        \E[\one[x \in D_m]]\phi_m}\\
        &\qquad + \ln(2M\log(\sum_{j=1^m} L_j)/\delta) \log (n)\phi_m  \\
    &\calT_{3,m} =  \one\left[\hinerr_m  >\frac{1}{\log(\sum_{j=1}^m) L_j \phi_m} \right]N_m 
\end{align*}
where $\hinerr_m = \hinerr(D_m)$.
Then, with probability at least $1-\delta$,  for any block $m$,
\begin{align*}
    \E_{\tau_m}&\left[\sum_{t = \tau_{m-1}+1}^{\tau_m} \E_t\Big[\one[y_t \text{ is queried}]\Big]\right] \\
    &\qquad \leq \order\left(\min\left\{ \calT_{1,m} + \calT_{2,m} + \calT_{3,m}, N_m \right\} \right)
\end{align*}
\end{theorem}

\paragraph{Discussion}
Theorem~\ref{thm:complex.block.m} shows that WL-AC preserves the advantage of original algorithm under the low quality weak labeler case. Indeed the expected query complexity in block $m$ is never larger than  $N_m$, i.e., the expected number of queries of \NWL-AC.

Ignoring $N_m$, the query complexity is characterizes by the terms $\calT_i$.
The term $\calT_1$ characterizes the number of queries we could save from using weak labels if their accuracy upper bound is known. 
In reality, however, the $\hinerr$ itself needs to be evaluated through samples. Therefore, the $\calT_2$ and $\calT_3$ capture the trades-off between the accurate WL evaluation and the WL-leveraged training sample collection.
%

While the first term in $\calT_2$ captures the required number of samples to evaluate the weak labels if no $P_{m,\min}$ lower bound exists, the second term explicitly write out the minimum label complexity of the original AC.  \footnote{Note that this lower bound $P_{m,\min}$ comes from the heavy tail term in  concentration inequality. We conjecture that it can be removed by using more robust estimators like Catoni estimator as in~\citep{https://doi.org/10.48550/arxiv.2003.01922}. Nevertheless, it is not clear how to simultaneously achieve efficiency so we left that as an open problem. }
The third term of $\calT_2$ captures the number of samples required to decide whether we need to switch to \NWL~mode which we give more detailed discussion in the next paragraph.
%

Finally, $\calT_3$  characterizes the label complexity in \NWL-mode, which depends on whether the conditional \WL error is worse than the true error of the best hypothesis restricted to the disagreement region.
Note that $\hinerr_m$ is always non-increasing with the shrinking of the disagreement region. This implies that the WL is helpful as long as it gives high accuracy on those informative samples, which aligns with the intuition in \cite{zhang2015active}, called Localized Difference Classifier Training.  Other than the theoretical implication, we want to remark that an explicit rule on whether we should stop using weak labeler is also good for practice.


\vspace{-5px}

\subsection{ More discussion on the total label complexity}
\label{sec: label complexity result total (main)}

Although our algorithm is guaranteed to always preserve the advantage of original AC in each block, the total complexity depends on the total number $M$ of blocks. Indeed, the total complexity of WL-AC is obtained by accumulating the query complexity of each block $m$. Several strategies can be used to control the number of blocks. For example, we can use a linear schedule or the classical block length doubling techniques. For ease of exposure, we consider $L_m = \sum_{j=1}^{m-1} L_j = 2^m$, which leads to $M = \order(\log n)$.


In general, our results is not directly comparable with previous work in~\citep{https://doi.org/10.48550/arxiv.1506.08669} because we are not in the realizable setting and it is hard to give an explicit form of $\Dot{P}_m(\cdot)$ as discussed in the original AC analysis. Nevertheless, here we give an intuitive discussion on their relation based on a relaxed upper bound. 

\begin{theorem}[Worst-case total label complexity under benign setting, informal]
\label{thm: total}

Choose $L_m = 2^m$. Suppose there exists an generalized WL error upper bound $\widetilde{\hinerr} \in [0,1]$ that
\begin{align*}
    \sum_{m=1}^M \calT_{1,m} + \calT_{3,m}
    \leq \widetilde{\hinerr} \sum_{m=1}^M N_m ,
\end{align*}
Then we have with probability at least $1-\delta$, 
\begin{align*}
    &\sum_{t=1}^n \one[y_t \text{ is queried}] \\
    &\leq  \order\biggl(\tilde{\theta} \sqrt{\berr_M(h^*)n\log(|\calH|n/\delta) + \log(|\calH|n/\delta)^2}\\
        & \quad + \theta^*\widetilde{\hinerr} (\berr_M(h^*)n + \log(|\calH|/\delta)) \biggr)
\end{align*}
where $\tilde{\theta} = \theta^*\log(n) + \sqrt{\theta^* \left(\sum_{m=1}^M \max\{\kappa(D_m),\frac{N_m}{\ln(2M\log(n)/\delta)}\}\right)\frac{\ln(\log(n)^2/\delta)}{ \log(|\calH|n/\delta)}}$
\end{theorem}

\begin{remark}
Note that in this bound we treat each $P_m$ in \NWL-AC as 1 and relax $\phi$ to its upper bound $\theta^*$.
\end{remark}

\paragraph{Comparison with the label complexity in difference-classifier-based algorithm \citep{https://doi.org/10.48550/arxiv.1506.08669}} 

The first term can be upper bounded as 
$\order\left( \Tilde{\theta}\log(\calH)\left(\frac{\err(h^*)}{\varepsilon} + 1\right) \right),$
where $\varepsilon$ is the target excess risk. This corresponds to the term $\otil\left(\theta^*\text{VCdim}(\calH_\text{diff})(\frac{\err(h^*)}{\varepsilon}+1)\right)$ in bound of \cite{https://doi.org/10.48550/arxiv.1506.08669} -- both describe the required samples to evaluate the weak labels.
Therefore, when the accuracy of weak labels has low variance ($\kappa(D_m)$ is small for most blocks), our algorithm can save up to $\frac{\text{VCdim}(\calH_\text{diff})}{\log|\calH|}$ in this term. Intuitively, this suggests that, when a lower complexity $\calH_\text{diff}$ compared to $\calH$ is given and a proper $h_\text{diff} \in \calH_\text{diff}$ exists, the previous one can give better result. Otherwise, our result is more preferable.

The second term can be upper bounded as
$\otil\left(\theta^*\widetilde{\hinerr}\log(|\calH|)\frac{\err(h^*)^2}{\varepsilon}^2\right)$, which corresponds to the term $\otil\biggl(\{\text{WL-and-$\theta_*$-related-term}\} *\text{VCdim}(\calH) \left(\frac{\err(h^*)^2}{\varepsilon^2}+1\right)\biggr) $
in the bound of \cite{zhang2015active}. Both describe the label complexity we could save by knowing weak label quality, but how these two WL quality characterization term related to each other remains unclear.

\section{Experiments}
\label{sec: experiment}

In this section, we show the effectiveness of the proposed practical version of WL-AC discussed in Section~\ref{sec: algo_practice (main)} through a set of experiments on corrupted MNIST.
%
%
The MNIST-C~\citep{Mu2019mnistc} is a comprehensive suite of 16 different types of corruptions applied to the MNIST dataset~\citep{lecun2010mnist}. Each contains 60000 training samples and 10000 test samples.
We consider a 2-layer convolutional neural network as hypothesis class in all the experiments.

\paragraph{\ding{182} Synthetic weak labels}
We choose one type of corruption (``impulse-noise'' corruption in main paper) as target task and generate synthetic weak labels purely based on the true labels. We further consider two processes for the generation of the synthetic weak labels.
In the first case denoted as noisy annotators (NA), we add uniform random noise across all the samples. That is, each sample has a certain probability to get the true label, otherwise it will get a random label from the other 9 classes. We show the robustness of our algorithm by testing on various levels of noise. In the second case denoted as localized classifiers (LC), we first do passive learning on the target task and then give correct weak labels to those samples only with high entropy, so the classifier is good on this local region. Thus we show the adaptiveness of our algorithm in the case where, although the overall quality of weak labelers are mediocre, the weak labeler is informative for those sample close to the decision boundary.  

\paragraph{\ding{183} Biased Pre-Trained Labelers} We train classifiers using different corruptions (``identity'', ``motion-blur'' and ``dotted-line'' in main paper) than the target task (``impulse noise'') and use these pre-trained classifiers to generate weak labels.



The overall objective of these experiments is to show the effectiveness of WL-AC compared to standard AL algorithms without weak labels and passive learning in realistic setting. 
%
The practical instantiation of WL-AC is described in Algorithm~\ref{algo: main (implement))} and can be used with any heuristic \NWL{} selection strategy. 
In our experiment, we first use uniform random sampling as baseline. Then we choose most commonly used query framework --  entropy-based (ET) uncertainty sampling, which is known to generalizing easily to probabilistic multi-label classifiers \citep{settles2009active}. Specifically, here we choose a hard-threshold for uncertainty sampling and only query samples above that threshold.
To conduct a thorough empirical study on how our algorithm performs with various AL methods remains to be a future direction.

We use ET(x) to denote entropy based uncertainty sampling with a threshold x and NA/LC(x) to denote the weak label setting where x is the noise level. Furthermore, we use ``passive'' to denote uniform sampling with the base NO-WL AL.

In each experiment, we average the results over 5 different random seeds. 
%
In each independent repetition, We randomly shuffle the whole training dataset, choosing 600 samples (1\% of the dataset) as validation set and the rest as training samples coming as stream. 
We set the model update block length as 1000.

\subsection{Results}


\begin{figure}[t]
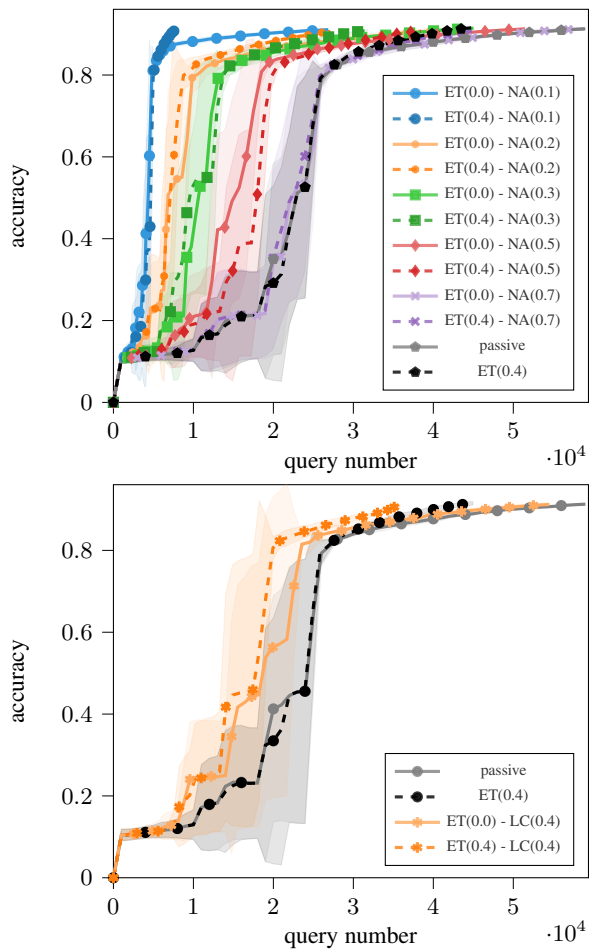

    \begin{subfigure}[b]{0.47\textwidth}
        \centering
        \input{plots/noisy_new.tex}
        \label{fig: noisy_query}
    \end{subfigure}
    \begin{subfigure}[b]{0.47\textwidth}
        \centering
        \input{plots/informative_new.tex}
        \label{fig: informative_queryl}
    \end{subfigure}
    \caption{\ding{182} - Performance comparison under uniform noisy (top) and informative classifier (bottom) weak labeler.}
    \label{fig:synt}
\end{figure}

We start considering experiment \ding{182}. Figure~\ref{fig:synt}(top) shows that our algorithm is able to leverage the weak labels when the noise level is low (below $0.5$) and greatly outperforms both the baseline (passive and ET(0) learners). The gain in query complexity is significant, up to a 80\% save in queries in the best case. In addition, this figure also show that our algorithm preserves the advantage of uncertainty sampling except the $0.5$ noisy case.  One possible reason is that $0.5$ is in the most ambiguous case -- if the noise ratio is low then we can easily take advantage of weak labels, if it is high then we can quickly detect that and switch to NOWL mode which can take advantage of uncertainty sampling. Finally, it is also worth to notice that Figure~\ref{fig:synt}(top) shows that our algorithm is robust to high levels of noise where it achieves similar performance as passive learning.

The uniform label noise fails to consider the case that the weak labels are more accuracy on those sample close to decision boundary. In fact, you may notice that accuracy increase rate becomes as flat as the passive learning in the later horizon, since the $\berr_m(f^*)$ decreases faster then $\hinerr_m$ and the algorithm ultimately switch to \textsc{NOWL} mode. Therefore, in Figure~\ref{fig:synt}(bottom), we test on this carefully designed informative classifier and further demonstrate the advantage of our algorithm in more structured weak labelers.
Without applying uncertainty sampling techniques, the advantage gained by leveraging weak labels will diminish when the active sampling set is shrinking. Only by incorporating this base AL strategy, the algorithm can get full advantage of the property of weak labeler. 
%

Figure~\ref{fig: pretrain_queryl} reports the results in scenario \ding{183}. Even in this more realistic experiment, similarly to the uniform noise experiment, WL-AC is able to leverage the weak labelers and reduce the query complexity to about 40\% of the one of the passive learner and 56\% of the one of the active learner. This implies not only the theoretical soundness of WL-AC, but the potential practical usage of our algorithm.
These experiments confirm the theoretical findings and show that WL-AC is effective in leveraging ``good'' hints while being robust to wrong hints.

\paragraph{Accuracy regarding to unlabeled sample complexity} We postpone the plots of the accuracy regarding to the number of observed unlabeled samples in Appendix~\ref{ sec: experiment unlabel results (supp)}, all of which show that our algorithm guarantee the similar accuracy as the passive one in term of the unlabeled samples.

\begin{figure}[t]
    \centering
    \input{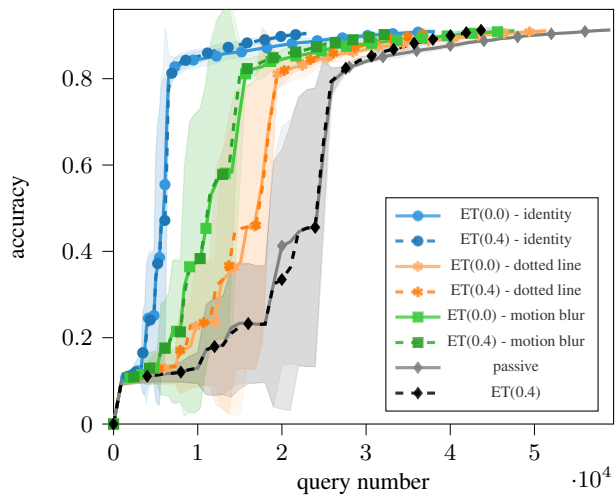}
    \caption{\ding{183} - Performance comparison under pre-trained classifier weak labeler.}
    \label{fig: pretrain_queryl}
\end{figure}

\section{Conclusions}
We introduced a novel algorithm, called WL-AC, for active learning with weak labelers that has both strong theoretical guarantees and good empirical performance. From a theoretical perspective, we showed that the query complexity of WL-AC is no worse than AC and, in certain problem instance, the access to weak labelers can lead to a significant reduction of the query complexity. Furthermore, we showed that WL-AC can be easily extended to practical scenarios where complex models (e.g., neural networks) are required. Our experiments on corrupted MNIST support the theoretical findings and show a significant improvement (more than 30\%) in query complexity compared to standard active learning methods (i.e., not using weak labels).

There are still many open questions. From the theoretical perspective, a natural combine our methods with the difference-classifier-based methods or find a better WL quality characterization that can establish a clear definition between these two. From the practical perspective, a more comprehensive numerical evaluation on incorporating our templates with various up-to-date AL in deep learning methods can further improve our understanding. 

\bibliographystyle{apalike}
\bibliography{ref}

\onecolumn
\aistatstitle{Appendix}
\tableofcontents
\newpage

\section{Appendix: Related works}
\label{sec: related works (supp)}

Active learning with multiple labeling sources has been studied in various setting. One common direction is to assume explicit structural labelers with various label cost and localized expertise, and then to adaptively select the proper labelers by modeling them during active learning. \citep{yan2012active, yan2014learning,fang2014active,huang2017cost}. But most of those works fail in providing rigorous statistically guarantees on their results and their strategy are mostly labeler-wise, which is too coarse compared to instance-wise studies as in \cite{zhang2015active}. Besides those empirical results, there are another line of theoretical works, which assumes the non-parametric settings. That is, those labelers are generally adhere to the rule that "similar samples have similar labels". \cite{urner2012learning,yan2016active}. Note that this notion can be viewed as one variants of our ``low local variance" WL assumption. Finally, all those works except the \cite{zhang2015active} we mentioned in main paper consider the pool-based setting, which is different from our streaming-based setting.

Online learning and bandits is another major topic in interactive learning literature which aims to minimize the regret regarding to best policy over time. As an counterpart of weak labeler setting in active learning, online learning with the help of loss predictors has been widely studied over the past decades. Many results have been achieved with full information feedback \citep{rakhlin2013online,steinhardt2014adaptivity}, as well as partial information feedback such as multi-armed bandits \cite{rakhlin2013online,wei2018more}. Our techniques are mainly inspired by the stochastic contextual bandits setting in the recent state-of-the-art result in \cite{https://doi.org/10.48550/arxiv.2003.01922}, specifically, we use the same doubly-robust estimator to reduce the loss estimator variance, and we borrow the high-level idea of starting an adaptive pure exploration phase for unknown hint quality (which is weak label quality in our setting).

\section{Appendix: Notations for analysis}
We define more notations here that will be used in the following analysis.
\begin{itemize}
    
    \item Let $\calD_*$ and $\calD_\WL$ be the underlying joint sampling distribution where each $(x,y)$, $(x,y^\WL)$ drawn from. 
    And denote the corresponding conditional probability distribution as $\fP_*(\cdot|x), \fP_\WL(\cdot|x)$. Most times, we will neglect the subscription in $\E$ and use that to the represent the expectation of all random variables inside the brackets. 
    \item The regret of a classifier $h \in \mathcal{H}$ relative to another $h^{\prime} \in \mathcal{H}$ is $\reg\left(h, h^{\prime}\right):=\operatorname{err}(h)-\operatorname{err}\left(h^{\prime}\right)$, and the analogous empirical regret on $S$ is $\reg\left(h, h^{\prime}, S\right):=\operatorname{err}(h, S)-\operatorname{err}\left(h^{\prime}, S\right)$. When the second classifier $h^{\prime}$ in (empirical) regret is omitted, it is taken to be the (empirical) error minimizer in $\mathcal{H}$.
    \item Define the expected regret associated to $\err(h,\Tilde{Z}_m)$ and $\reg(h,h',\Tilde{Z}_m)$. \\
    \begin{align*}
        \reg_m^{\ddagger}\left(h, h^{\prime}\right):=& \E_X\left[\left(\one\left(h(X) \neq h_m(X)\right)-\one\left(h^{\prime}(X) \neq h_m(X)\right)\right) \one\left(X \notin D_m\right)\right]+\\
        & \E_{X, Y}\left[\left(\one(h(X) \neq Y)-\one\left(h^{\prime}(X) \neq Y\right)\right) \one\left(X \in D_m\right)\right], \\
        \widetilde{\reg_m}\left(h, h^{\prime}\right)
        :=& \frac{1}{\sum_{j=1}^m L_j} \sum_{j=1}^m L_j \reg_j^{\ddagger}\left(h, h^{\prime}\right) .
    \end{align*}
    \item Let $Z_m$ be the subset of $\Tilde{Z}_m$ that are actually being queried. 
    \item In some proof step, we use $\approx$ to hide the constant factor since we only care about the order.
    \item For convenience, we use $\sum_Z$ to denote $\sum_{(x,y,y^\WL,w) \in Z}$ for any collection dataset $Z$.
    \item We use $\E_t[\cdot]$ to denote the expectation at time $t$ (or sample $t$) condition on the past. 
    \item We use $m(i)$ to denote the block any $i$-th samples located.
    \item Let  $\Dot{N}_m = \E\left[\sum_{t =\tau_m}^{\Dot{\tau}_m} \one[ y_t \text{ is queried}] \right]$, which is the expected query number in Phase 1 (Algo.~\ref{algo: hint-eval}) in any block $m$.
    
\end{itemize}
\section{Appendix: WL-AC (Theoretical version)}
\label{sec: algo_theory (supp)}

Here we present a detailed theoretical version of WL-AC under the templates in Algo.~\ref{algo: main}

\begin{algorithm}[H]
\small
\caption{WL-AC (theoretical version)}
\label{algo: main-inefficient (supp)}
\begin{algorithmic}[1]
\STATE \textbf{Input: }Constants $c_1, c_2, c_3$, confidence $\delta$, error radius $\gamma$, parameters $\alpha, \beta, \xi$ for (OP).  Data collection scheduled block length $L_1,L_2,\ldots$ satisfying $L_{m+1} \leq \sum_{j=1}^m L_j$ and $L_1 = 3$. Strong labeler $\calO$ and weak labeler $\calW$. Candidate hypothesis set $\calH$. 
\STATE \textbf{initialize:} epoch $m=0, \tilde{Z}_{0}:=\emptyset, \Dot{Z}_0: = \emptyset, \Delta_{0}:=c_{1} \sqrt{\epsilon_{1}}+c_{2} \epsilon_{1} \log 3$, where
$$
\epsilon_{m}:=\frac{32\left(\log (|\mathcal{\calH}| / \delta)+\log(\sum_{j=1}^m L_j)\right)}{\sum_{j=1}^m L_j}.
$$

\FOR{ $m = 1,2,\ldots, M$}
    \STATE \textcolor{orange}{Phase 1: WL evaluation and query probability assignment}
    \STATE Compute the \NWL-leveraged query probability $P_{m}= \textbf{\textsc{OP}}(\textsc{Use-WL} = False$) in \ref{algo: op} 
    \STATE Call \textbf{\textsc{WL-Eval}} in Algo.~\ref{algo: hint-eval}, and obtain the current \WL~evaluation dataset $\Dot{Z}_m$, the flag $\textsc{Use-WL}_m$ which decides whether we should transit to NO-WL mode or not and the pessimistically estimated conditional WL error $\Dot{\hinerr}_m$ across $D_m$. 
    \STATE Compute the WL leveraged query probability $\Dot{P}_m(\cdot) = \textbf{\textsc{OP}}(\textsc{Use-WL}_m, \Dot{\hinerr}_m)$
    \STATE Set the stopping time of Phase 1 as $\Dot{\tau}_m$ and the corresponding length as $\Dot{L}_m$
    \STATE \textcolor{orange}{Phase 2: Train data collection based on planed query probability}
    \STATE Set $S = \emptyset$
    \FOR{$t = \Dot{\tau}_m +1, \ldots, \Dot{\tau}_m + L_m$}
        \IF{$x_t \in D_m$}
            \STATE Draw $Q_t \sim \text{Bernoulli}(\Dot{P}_m(x))$ and query label $y_t$ if $Q_t=1$
            \STATE Update the set of examples:
            \begin{align*}
                S:=
                \begin{cases}
                    S \cup\left\{\left(x_t, y_t, y^\WL_t,1 / \Dot{P}_{m}\left(x_t \right)\right)\right\}, & Q_t=1 \\
                    S \cup\left\{x_t, 1,1,0\right\}, & \text { otherwise. }
                 \end{cases}
            \end{align*}
        \ELSE
            \STATE $S:=  S \cup\left\{\left(x_t, h_m\left(x_t\right),h_m\left(x_t\right), 1\right)\right\} .$ 
        \ENDIF
    \ENDFOR
    \STATE Set $\tilde{Z}_m = \tilde{Z}_{m-1} \cup S$
    \STATE \textcolor{orange}{Phase 3: Model and disagreement region updates using collected train data }
    \STATE Calculate the estimated error for all $h \in \calH$ using shifted double robust estimator
        \begin{align*}
            \err(h,\tilde{Z}_m) =
            \begin{cases}
            \sum_{(x,y,y^\WL,w) \in \tilde{Z}_m} \left(\one[h(x) \neq y] - \one[h(x) \neq y^\WL]\right) w + \one[h(x) \neq y^\WL] \quad &\textsc{USE-WL}_m = True\\
            \sum_{(x,y,y^\WL,w) \in \tilde{Z}_m} \one[h(x) \neq y]w & \NWL
            \end{cases}
        \end{align*}
    \STATE Calculate the empirical best hypothesis $h_{m+1}: = \argmin_{h \in \calH} \err(h,\tilde{Z}_m)$
    \STATE Update the active hypothesis set (hypothesis that always include the best hypothesis)
    \begin{align*}
        & \Delta_m : = c_1 \sqrt{\epsilon_m \err(h,\tilde{Z}_m)} + c_2 \epsilon_m \log (\sum_{j=1}^m L_j) \\
        & A_{m+1} : = \left\{ \err(h,\tilde{Z}_m) - \err(h_{m+1},\tilde{Z}_m) \leq \gamma \Delta_m \right\}\\
        & D_{m+1} : = \dis( A_{m+1}) = \{ x \in \calX | \exists h,h' \in A_{m+1}, h(x) \neq h'(x)\}
    \end{align*}
    \STATE Set the end time of the current block as $\tau_{m} = t = \Dot{\tau}_m + L_m$
\ENDFOR
\RETURN $h_M$
\end{algorithmic}
\end{algorithm}
Note that we write out the Line~\ref{line: wl-eval} in the templates into more details, instead of merging into the single WL-EVAL. The various constants in Algo\ref{algo: main-inefficient (supp)} must satisfy:\\
\begin{equation*}
\begin{gathered}
\alpha \geq 1, \quad \eta \geq 864, \quad \xi \leq \frac{1}{8 n \epsilon_M \log n}, \quad \beta^2 \leq \frac{\eta}{864 \gamma n \epsilon_M \log n}, \quad \gamma \geq \eta / 4 \\
c_1 \geq 6 \alpha, \quad c_2 \geq \eta c_1^2 / 2 + 13, \quad c_3 \geq 1
\end{gathered}
\end{equation*}

\begin{algorithm}[H]

\small
\caption{\WL~evaluation for block $m$ (WL-EVAL, theoretical version)}
\label{algo: hint-eval}
\begin{algorithmic}[1]
\STATE \textbf{Input: } $D_m$, $\err(h_m, \Tilde{Z}_{m-1})$,$\Delta_{m-1}$, $L_m$, the current collected \WL~evaluation samples $\Dot{Z}_{m-1}$ and the query probability without \WL~denoted as $P_m$
\STATE Set $\berr(h_m, \Tilde{Z}_{m-1}) = \err(h_m, \Tilde{Z}_{m-1}) + \Delta_{m-1}$
\STATE \textcolor{orange}{Step1: Check if the current estimated best error is already good enough}
\STATE Set $N_m = \E[P_m(x)]$, which is the expected query number of block $m$ without leveraging the WL \label{line: N_m}
\IF{$N_m \leq \left(\frac{6*2\ln(2M\log(n)/\delta)}{\berr(h_m, \Tilde{Z}_{m-1})}   - |\Dot{Z}_{m-1}|\right) \E[\one[x \in D_m]]$  }
    \RETURN $\textsc{Use-WL} = \text{False}, \Dot{\hinerr}_m = 1, \Dot{Z}_m = \Dot{Z}_{m-1}$
\ENDIF  
\STATE \textcolor{orange}{Step2: Check if the WL performance is worse than $\berr(h_m, \Tilde{Z}_{m-1})$}
\STATE Initialize $k = K_m^s =  \lceil\log(\frac{6}{ \berr(h_m, \Tilde{Z}_{m-1})}\rceil$ and initialize $\Dot{S} = \emptyset$ 
\STATE Get $\kappa_m = \kappa(D_m)$ from $\calW$ as defined in \ref{ass: hint oracle}
\STATE Draw $\max\{(2^{k+1}\ln(2M\log(n)/\delta)-|\Dot{Z}_{m-1}| ),0\} $ number of new unlabeled samples and query labels for those inside $D_m$. For each sample at $t$, add them as 
    \begin{align*}
        \Dot{S}:=
        \begin{cases}
            \Dot{S} \cup\left\{x_t, y_t, y^\WL_t\right\}, & Q_t=1 \\
            \Dot{S} \cup\left\{x_t, 1,1\right\}, & \text { otherwise. }
         \end{cases}
    \end{align*}
and update $\ddot{Z}_{m,k} = \Dot{Z}_{m-1} \cup \Dot{S}$
\STATE Calculate the empirical mean $\wh{\hinerr}_{m,k} = \frac{1}{|\ddot{Z}_{m,k}|} \sum_{(x,y,y^\WL) \in \ddot{Z}_{m,k}} \one[y^\WL \neq y \wedge x \in D_m]$ 
\STATE Calculate the pessimistic overall estimation $\overline{\hinerr}_{m,k} =  \min\{\wh{\hinerr}_k + 2^{-k}, \E[x \in D_m] \}$ 
\IF{$\overline{\hinerr}_k \geq \berr(h_m, \Tilde{Z}_{m-1})$}
    \RETURN $\textsc{Use-WL} = \text{False}, \Dot{\hinerr}_m = 1,  \Dot{Z}_m = \ddot{Z}_{m,k}$
\ENDIF
\STATE \textcolor{orange}{Step3: Get a more precise estimation on WL performance}
\IF{$\max\{4N_m  \min\left\{\frac{ \kappa_m \overline{\hinerr}_{m,k}}{\E[\one[x \in D_m]]},1\right\}, L_m P_{m,\min}\E[\one[x \in D_m]]\} \leq \E[x \in D_m]2^{k+1}\ln(2M\log(n)/\delta)$ }
    \RETURN $\textsc{Use-WL} = \text{True}, \Dot{\hinerr}_m =  \min\left\{\frac{ \kappa_m \overline{\hinerr}_{m,k}}{\E[\one[x \in D_m]]},1\right\},  \ddot{Z}_{m,k}$ 
\ENDIF
\FOR{$k = K_m^s+1,\ldots$}
\STATE Initialize $\Dot{S} = \emptyset$ 
\STATE Draw $\max\{(2^{k+1}\ln(2M\log(n)/\delta)-|\ddot{Z}_{m,k-1}| ),0\} $ number of new unlabeled samples and query labels for those inside $D_m$. For each sample at $t$, add them as 
    \begin{align*}
        \Dot{S}:=
        \begin{cases}
            \Dot{S} \cup\left\{x_t, y_t, y^\WL_t\right\}, & Q_t=1 \\
            \Dot{S} \cup\left\{x_t, 1,1\right\}, & \text { otherwise. }
         \end{cases}
    \end{align*}
and update $\ddot{Z}_{m,k} = \ddot{Z}_{m,k-1} \cup \Dot{S}$
\STATE Calculate the the optimistic estimation $\overline{\hinerr}_{m,k}$ using $\ddot{Z}_{m,k}$ as before
\IF{$\max\{4N_m \min\left\{\frac{ \kappa_m \overline{\hinerr}_{m,k}}{\E[\one[x \in D_m]]},1\right\}, L_m P_{m,\min}\E[\one[x \in D_m]]\} \leq \E[x \in D_m]2^{k+1}\ln(2M\log(n)/\delta)$ }
    \RETURN $\textsc{Use-WL} = \text{True}, \Dot{\hinerr}_m = \min\left\{\frac{ \kappa_m \overline{\hinerr}_{m,k}}{\E[\one[x \in D_m]]},1\right\},  \ddot{Z}_{m,k}$ 
\ENDIF
\ENDFOR
\end{algorithmic}
\end{algorithm}

As described in the templates, there are two purposes for WL-EVAL: \textbf{(1)} Deciding whether to transit to \NWL~mode by comparing the WL error with estimated $\berr(h_m, \tilde{Z}_{m-1})$; \textbf{(2)} Pessimistically estimating the conditional WL error if USE-WL. To achieve this, we have three steps. In Step 1, we compute the required number to estimate $\E[ \one[y^\WL \neq y | x] \one[ x \in D_m]]$ with sufficient accuracy to compare with $\berr(h_m, \tilde{Z}_{m-1})$. If such number is too large, which indicates  $\berr(h_m, \tilde{Z}_{m-1})$ has already been sufficiently small, then we transit to \NWL~mode without collecting any samples in WL-EVAL. Otherwise, we collect the required number of samples and do the comparison with $\berr(h_m, \tilde{Z}_{m-1})$ in Step 2. If we pass the comparison, then we further estimate the conditional WL error in Step 3.

\newpage
\begin{mdframed}
\label{algo: op}
\small
{\centering \textbf{Optimization Problem} (OP, theoretical version) to compute $P_m$ \\}
\textbf{Input: } Estimated WL performance upper bound $\Dot{\hinerr}_m$, $\textsc{use-WL}$
\begin{align*}
    &\min_{P} \quad \E_{X}\left[\frac{1}{1-P(X)}\right]\\
    \text { s.t. } \quad 
    &\forall h \in \mathcal{
    \calH}, \quad 
    \begin{cases}
        \E_{X}\left[\frac{\one\left(h(x) \neq h_m(x) \wedge x \in D_{m}\right)}{P(X)}\right] \Dot{\hinerr}_m \leq b_{m}^\WL(h) 
        & \textsc{Use-WL}\\
        \E_{X}\left[\frac{\one\left(h(x) \neq h_m(x) \wedge x \in D_{m}\right)}{P(X)}\right] \leq b_{m}^{\NWL}(h) 
         & \text{otherwise}
    \end{cases}
    \\
    \text { where } \quad &\one_{h}^{m}(x) =\one\left(h(x) \neq h_m(x) \wedge x \in D_{m}\right), \\
    &b_{m}^\NWL(h) =2 \alpha^{2} \E_{x}\left[\one_{h}^{m}(x)\right]+2 \beta^{2} \gamma \reg\left(h, h_m, \tilde{Z}_{m-1}\right) \tau_{m-1} \Delta_{m-1}+\xi \tau_{m-1} \Delta_{m-1}^{2}, \\
    &b_{m}^\WL(h) = \frac{1}{2}(b_{m}^\NWL(h)- \E_{X}\left[\one_{h}^{m}(x)\right]) \\
    &P_{\min , m} =\min \left(\frac{c_{3}}{\sqrt{\frac{(\sum_{j=1}^{m-1}L_m) \err\left(h_m, \tilde{Z}_{m-1}\right) }{n\epsilon_{M}}}+\log(\sum_{j=1}^{m-1}L_m)}, \frac{1}{2}\right)
\end{align*}
\end{mdframed}

\begin{remark}[Efficient implementation]
As discussed in \cite{https://doi.org/10.48550/arxiv.1506.08669}, this theoretical version of WL-AC algorithm itself can actually be implemented efficiently by solving $\one[x \in D_m]$ with ERM oracle and estimating the expected distribution with empirical distribution under the small scale models. 
\end{remark}

\subsection{A more detailed explanation on term $N_m$ and its relationship with NOWL-AC}

Readers may notice that we abuse the notation $N_m$ a little bit. Here we define $N_m = \E[ P_m(x) | \text{past history in WL-AC}]$ in Line~\ref{line: N_m} while in Section~\ref{sec: result(main)} we define $N_m = \E[P_m(x) | \text{past history in NOWL-AC}]$. It is hard to argue that how exactly these two are closed to each other because each OP at block $m$ depends on $\reg(h,h_{m}, \tilde{Z}_{m-1})$ which is effected by the randomness during querying in Phase 2. Nevertheless, by the generalization guarantees shown in Theorem~\ref{them: general}, we can guarantee that these two are roughly in the same order. Specifically, from the theorem we have for any $h$,
\begin{align*}
    & \reg(h,h_{m},\tilde{Z}_{m-1}) \leq \frac{3}{2} \widetilde{\reg}_m(h,h^*) + \frac{\eta}{2}\Delta_{m-1} \\
    &  \reg(h,h_{m},\tilde{Z}_{m-1}) \geq\frac{1}{2} \widetilde{\reg}_m(h,h^*) - \frac{\eta}{4}\Delta_{m-1},
\end{align*}
no matter whether this $\tilde{Z}_{m-1}$ comes from NOWL-AC or WL-AC. Therefore, replace this upper bound and lower bound into the NOWL mode of WL-AC and the corresponding constraints in NOWL-AC, we have roughly
\begin{align*}
    b_m^\NWL 
    =  \widetilde{\Theta}(\E_{x}\left[\one_{h}^{m}(X)\right] + \widetilde{\reg}_m\left(h, h^*\right) \tau_{m-1} \Delta_{m-1} +  \tau_{m-1} \Delta_{m-1}^2),
\end{align*}
regardless of the randomness in query, in both WL-AC version and its NOWL-AC counterparts.  Thus this $N_m$ in algorithm is in the same order as we define in main paper, and therefore preserves benefits in original AC.

We further give an example raised from the original AC to demonstrate how we keep the parismoniousness property in some benign cases.

\paragraph{Example restatement}

 Let $U\{-1,1\}$ denote the uniform distribution on $\{-1,1\}$. The data distribution $\mathcal{D}(\mathcal{X}, \mathcal{Y})$ and the classifiers are defined jointly:
 \begin{itemize}
     \item With probability $\epsilon$,
        $y=h^*(x), \quad h(x) \sim U\{-1,1\}, \forall h \neq h^* .$
    \item With probability $1-\epsilon$, 
        \begin{align*}
            & y \sim U\{-1,1\}, \quad h^*(x) \sim U\{-1,1\},\\
            & h_r(x)=-h^*(x) \quad \text{ for some $h_r$ drawn uniformly at random from } \mathcal{H} \backslash h^*,\\
            & h(x)=h^*(x),  \forall h \neq h^* \wedge h \neq h_r.
        \end{align*}
 \end{itemize}
Under this setting, the best classifier has $\err(h^*) = \frac{1-\epsilon}{2}$ and others has $\err(h) = \frac{1}{2}$. It is easy to see that only $\epsilon$ case is informative for us to distinguish $h^*$. The traditional disagreement based algorithm will waste a lost of queries on uninformative sample due to the slow shrinking disagreement region.

Now in \cite{https://doi.org/10.48550/arxiv.1506.08669}, they show that $P_m(x) = P_{\min,m}, \forall x \in D_m$ is a feasible solution when $|\calH|$ is large because this choice of $P(x)$ satisfies
\begin{align*}
    \E_{X}\left[\frac{\one\left(h(x) \neq h_m(x) \wedge x \in D_{m}\right)}{P(X)}\right] \leq \order(\xi) \leq \order(\xi \tau_{m-1} \Delta_{m-1}^{2}). \text{ (see details in original paper)}
\end{align*}

The same argument works also in our case, because $\tau_{m-1}$, $\Delta_{m-1}$ does not affected by the randomness of sampling. That is,
\begin{align*}
    \E_{X}\left[\frac{\one\left(h(x) \neq h_m(x) \wedge x \in D_{m}\right)}{P(X)}\right] \Dot{\hinerr}_m 
    &\leq \E_{X}\left[\frac{\one\left(h(x) \neq h_m(x) \wedge x \in D_{m}\right)}{P(X)}\right] \\
    & \leq \order(\xi)
    \leq \frac{1}{2} \xi \tau_{m-1} \Delta_{m-1}^{2}
    \leq b_m^\WL(h)
\end{align*}

\section{Appendix: Guarantees of pessimistic estimation \WL error}

Define the event that that the error \WL has been estimated within proper confidence region


\begin{align*}
    &\calE_\WL\\
    &:= \left\{ \wh{\hinerr}_{m,k} \in \left[\frac{1}{2}\E_{x,y,y^\WL} \one[x \in D_{m} \wedge y^\WL \neq y ] -\frac{2^{-k}}{2},2\E_{x,y,y^\WL} \one[x \in D_{m} \wedge y^\WL \neq y ] + 2^{-k}\right], \forall k \leq K_m^e,  \forall m \in [M] \right\}
\end{align*}

\begin{lemma}
    By running Algo.~\ref{algo: hint-eval} in each block, $\Pr[\calE_\WL] \geq 1-\delta$.
\end{lemma}
\begin{proof}
By using Bernstein inequality and the union bound over all blocks, we have with probability $1-\delta/2$, for all block $m$ and $k \in [K_m^s,K_m^e]$,  
\begin{align*}
    &\frac{1}{|\dot{Z}_{m,k}|}\sum_{ (x,y,y^\WL) \in \dot{Z}_{m,k}}\one[y^\WL \neq y \wedge x \in D_{m}]
    - \E[\one[y^\WL \neq y \wedge x \in D_{m}]]\\
    & \leq \sqrt{\frac{2 \E[\one[y^\WL \neq y \wedge x \in D_{m}]]\ln(4M\log(n)/\delta)}{|\dot{Z}_m|}} + \frac{3 \ln(2M \log(n)/\delta)}{|\dot{Z}_m|}
\end{align*}
On the other hand, by applying empirical Bernstein inequality and the union bound,  we have with probability $1-\delta/2$, for all $k,m$
\begin{align*}
    & \E[\one[y^\WL \neq y \wedge x \in D_{m}]] - \frac{1}{|\Dot{Z}_{m,k}|}\sum_{  (x,y,y^\WL) \in \Dot{Z}_{m,k}}\one[y^\WL \neq y \wedge x \in D_{m}]\\
    & \leq  \sqrt{\frac{\frac{2}{|\Dot{Z}_{m,k}|}\sum_{  (x,y,y^\WL) \in \Dot{Z}_{m,k}}\one[y^\WL \neq y \wedge x \in D_m]\ln(4M\log(n)/\delta)}{|\Dot{Z}_{m,k}|}} + \frac{3 \ln(2M \log(n)/\delta)}{|\Dot{Z}_{m,k}|}
\end{align*}
Now by the choice $|\Dot{Z}_{m,k}| = 2^{k+1}\ln(2 M \log(n)/\delta)$, we get the desired bound.
\end{proof}

\begin{corollary}
When $\calE_\WL$ holds, for any m and $k \leq K_m^e$, $\overline{\hinerr}_{m,k} = \min\{ 2\wh{\hinerr}_m + 2^{-k},1\}$ is an pessimistic estimation (upper bound) of $\E\left[\one[x \in D_m \wedge y^\WL \neq y ]\right]$. 
\end{corollary}

By using this assumption, we get the following WL quality estimation guarantees.
\begin{lemma}[Guarantee on the estimated conditional error]
\label{lem: conditional hint error est}
Under the assumption~\ref{ass: hint oracle}, when $\calE_\WL$ holds, we have for any m and $k \in [K_m^s, K_m^e]$,
\begin{align*}
    \hinerr(D_m)
   \leq \underbrace{\min\left\{\kappa(D_m) \frac{2\wh{\hinerr}_{m,k} + 2^{-k}}{\E[\one[x \in D_m]]},1 \right\}}_{\min\left\{\kappa(D_m)\frac{\overline{\hinerr}_{m,k}}{\E[\one[x \in D_m]]},1\right\}}
    \leq \min\left\{\kappa(D_m) \left( 4\hinerr(D_m) + 3\frac{2^{-k}}{\E[\one[x \in D_m]]} \right),1\right\}
\end{align*}
Notice here $\min\left\{\kappa(D_m)\frac{\overline{\hinerr}_{m,k}}{\E[\one[x \in D_m]]},1\right\}$ is the pessimistic estimation of the \textit{conditional} WL error.
\end{lemma}
\begin{proof}
Recall by definition in Assumption~\ref{ass: hint oracle}, $\hinerr(D_m) =  \max_{x \in D} \E[\one[y^\WL \neq y]|x]$. Therefore we have,
\begin{align*}
     \max_{x \in D} \E[\one[y^\WL \neq y]|x]
     & \leq  \kappa(D_m) \min_{x \in D_m} \E[\one[y^\WL \neq y]|x]\\
     & \leq  \kappa(D_m) \frac{\E[\one[y^\WL \neq y \wedge x \in D_m]]}{ \E[\one[x \in D_m]]}\\
     & \leq  \kappa(D_m) \frac{ 2\wh{\hinerr}_m + 2^{k}}{ \E[\one[x \in D_m]]} 
    \\
     & \leq  \kappa(D_m) \frac{4\E[\one[y^\WL \neq y \wedge x \in D_m]] + 3*2^{-k}}{ \E[\one[x \in D_m]]} \\
     & \leq  \kappa(D_m) \frac{4\E[\one[x \in D_m]]\max_{x \in D_m} \E[\one[y^\WL \neq y]|x]  + 3*2^{-k}}{ \E[\one[x \in D_m]]}\\
     & \leq  4\kappa(D_m) \max_{x \in D} \E[\one[y^\WL \neq y]|x] + 3\kappa(D_m) \frac{2^{-k}}{\E[\one[x \in D_m]]}
\end{align*}
where the third and forth inequality comes from $\calE_\WL$.
\end{proof}
\section{Appendix: Generalization guarantees and its analysis}
\label{sec: generalization result (supp)}

\subsection{Relation with the original proofs without WL evaluation in Phase 1 (Read before going to analysis)}

Before start proving the generalization guarantees, we want to remark that our analysis is built on the original proof of \cite{https://doi.org/10.48550/arxiv.1506.08669}, so we will skip most of the proofs that are exactly the same as original one and only give analysis on those significant modifications. In order to make readers easy to connect our proofs with the original one, without loss of generality, we will ignore the samples collected in Phase 1, that is $\Dot{\tau}_m = \tau_m$ in the rest of sections, and therefore, $L_m = \tau_m - \tau_{m-1}$. In the other word, within this section, readers can regard the pessimistic estimated $\Dot{\hinerr}$ directly given by some oracle.

One subtlety to make such simplification is that, our final generalization guarantees will only depend on the number of total unlabeled data in Phase 2, that is $\{L_m\}_{m=1}^M$, but does not include the number of unlabeled data in phase 1, that is $\Dot{\tau}_m - \tau_{m-1}\}_{m=1}^M$. In order to make our excess risk result depends on the whole unlabeled sample complexity $n$ so that we can compare it with the original AC, we show that $L_m$ and $\Dot{\tau}_m - \tau_{m-1}$ are roughly as the same order. Therefore, for all the $\tau_m$ notation in the rest of analysis, after considering the Phase 1, we need to replace $\tau_m$ with $\order(\tau_m)$ (see Lemma~\ref{lem: Unlabeled sample guarantee}, which does not effect the order of those number-of-unlabeled-samples dependent results.)

\subsection{Deviation bounds and the benefits of using shifted double robust estimator }

We first show the deviation bound with or without the shifted double robust estimator. This bound is different than the original one in \cite{https://doi.org/10.48550/arxiv.1506.08669} but later we will show that we can control the deviation of the empirical regret and error terms by combining similar techniques in original paper and our carefully designed hint leveraging strategies.
\\\\
For convenience, let's let's first define the instantaneous variance of regret and error in expectation for any fixed block $m$ and any pair of hypothesis $h,h'$ (or a single hypothesis $h$) as follows
\begin{align*}
    & \regvar_\WL(h,h',m) =  \E\left[\left(\frac{2}{\Dot{P}_m(x)} \one[y^\WL \neq y] + 1\right)\one[h(x) \neq h'(x) \wedge x \in D_m] + \one[h(x) \neq h'(x) \wedge x \notin D_m]\right], \\
    & \regvar_\NWL(h,h',m) = \E\left[ \frac{1}{\Dot{P}_m(x)} \one[h(x) \neq h'(x) \wedge x \in D_m]  + \one[h(x) \neq h'(x) \wedge x \notin D_m]\right],\\
    & \errvar_\WL(h,m) = \E\left[ \frac{\one\left(x \in D_{m} \wedge y^\WL \neq y\right)}{\Dot{P}_m(x)}+ \one\left( x \in D_m \wedge h(x) \neq y\right)\right] \\
    & \errvar_\NWL(h,m) = \E\left[\frac{\one\left(x \in D_{m} \wedge h(x) \neq y\right)}{\Dot{P}_m(x)}\right],
\end{align*}
Note that subscript $\WL$ denotes the \textsc{Use-WL} mode where double robust estimator is used, otherwise the normal inverse weighted estimator is used. Again as we discussed in the first contribution in Sections~\ref{sec: algo (main)} that, the weak labeler will never hurt estimating difference between two policies (up to some constant error) since $\one[y^\WL \neq y] \leq 1$. But it will hurt estimating the error for a single classifier. Specifically, we have,

\begin{lemma}[Deviation bounds]
\label{lem: deviation bound}
 Pick $0<\delta<1 /$ e such that $|\mathcal{H}| / \delta>\sqrt{192}$. With probability at least $1-\delta$ the following holds for all $m \geq 1$. 
\begin{align*}
    &\left|\reg\left(h, h',\tilde{Z}_{m}\right)-\tilreg_{m}(h,h')\right| \\
    &\leq \sqrt{\frac{\epsilon_{m}}{\tau_{m}} \left(\sum_{ j \in M_{m,\WL}} (\tau_{j}-\tau_{j-1})\regvar_\WL(h,h',j) + 
    \sum_{ j \in M_{m,\NWL}} (\tau_{j}-\tau_{j-1}) \regvar_\NWL(h,h',j) \right)}+\frac{2\epsilon_{m}}{\Dot{P}_{\min , m}}\\
    &\left|\err\left(h, Z_{m}\right)-\overline{\err}_{m}(h)\right| \\
    &\leq \sqrt{\frac{\epsilon_{m}}{\tau_{m}} \left(\sum_{ j \in M_{m,\WL}} (\tau_{j}-\tau_{j-1}) \errvar_\WL(h,j) + 
    \sum_{ j \in M_{m,\NWL}} (\tau_{j}-\tau_{j-1})  \errvar_\NWL(h,j) \right)}+\frac{\epsilon_{m}}{\Dot{P}_{\min , m}}
\end{align*}
where $M_{m,\WL} = \left\{j| j\leq m, \textsc{Use-WL}_j = True \right\}$  and vice versa.

\textcolor{blue}{To compare, the original bounds only have \NWL~terms.}
\end{lemma}

\begin{proof}
We will only prove the \WL~mode since the proof for \NWL~is exactly the same as the original proof.

Again our proof follows the similar steps as in the proof of Lemma 2 in the original paper with a specialized adaption to our shifted loss estimator. First we look at the concentration of the empirical regret on $\tilde{Z}_m$. To avoid clutter, we overload our notation so that $D_i=D_{m(i)}$, $h_i=h_{m(i)}$ and $P_i=P_{m(i)}$ when $i$ is the index of an example rather than a round.

For any pair of classifier $h$ and $h'$, we define the random variables for the instantaneous regrets:
\begin{align*}
    \tilde{R}_i
    & : = \one[x_i \in D_i] \biggl[\left( \frac{\one[h(x_i) \neq y_i] -\one[h(x_i) \neq y_i^\WL] }{P_i(x_i)} Q_i + \one[h(x_i) \neq y_i^\WL] \right) \\
         &\quad - \left( \frac{\one[h'(x_i) \neq y_i] -\one[h'(x_i) \neq y_i^\WL] }{P_i(x_i)}Q_i + \one[h'(x_i) \neq y_i^\WL] \right)\biggr]\\
         & \quad +  \one[x_i \notin D_i]\left(\one[h(x_i) \neq h_i(x_i)] - \one[h'(x_i) \neq h_i(x_i)]\right)
\end{align*}
and the associated $\sigma$-fields $\calF_i:=\sigma\left(\left\{X_j, Y_j, Q_j\right\}_{j=1}^i\right)$. We have that $\tilde{R}_i$ is measurable with respect to $\calF_i$. Therefore $\tilde{R}_i-\E\left[\tilde{R}_i \mid \calF_{i-1}\right]$ forms a martingale difference sequence adapted to the filtrations $\calF_i, i \geq 1$, and $x_i,y_i,Q_i$ are independent from the past. Now we want to state several properties of the instantaneous regrets.

First of all, we can simply the instantaneous regrets as 
\begin{align*}
    \tilde{R}_i
    & = \one[x_i \in D_i] \one[y_i = y^\WL] (\one[h(x_i) \neq y_i^\WL] - \one[h'(x_i) \neq y_i^\WL])\\
        & \quad + \one[x_i \in D_i] \one[y_i \neq y^\WL] \frac{Q_i}{\dot{P}_i(x_i)} \left((\one[h(x_i) \neq y_i]- \one[h'(x_i) \neq y_i])  - (\one[h(x_i) \neq y_i^\WL]- \one[h'(x_i) \neq y_i^\WL])\right) \\
        &\quad + \one[x_i \in D_i] \one[y_i \neq y^\WL]  (\one[h(x_i) \neq y_i^\WL] - \one[h'(x_i) \neq y_i^\WL]) \\
        &\quad + \one[x_i \notin D_i]\left(\one[h(x_i) \neq h_i(x_i)] - \one[h'(x_i) \neq h_i(x_i)]\right)
\end{align*}
Therefore, we can get the upper bound of magnitude, expectation and variance as 
\begin{align*}
    & |\tilde{R}_i - \E[\tilde{R}_i | \calF_{i-1}]| \leq \frac{2}{\dot{P}_i(x_i)} \leq \frac{2}{P_\text{min,m}}  \text{ for all } m \geq m(i) \\
    & \E[\tilde{R}_i | \calF_{i-1}] = \one[x_i \in D_i] \left( \one[h(x_i) \neq y_i] - \one[h'(x_i) \neq y_i]\right)
        + \one[x_i \notin D_i]\left(\one[h(x_i) \neq h_i(x_i)] - \one[h'(x_i) \neq h_i(x_i)]\right)\\
    & \E\left[\left(\tilde{R}_i - \E[\tilde{R}_i | \calF_{i-1}]\right)^2 |\calF_{i-1}\right]
    \leq  \regvar_\WL(h,h',m(i))
\end{align*}
Therefore, by applying a variant of freedman inequality (Lemma 6 in original paper) we get the desired bound. (For more details please refer to the explanation on eqn.(51) in the original paper.)

Then we consider the concentration of the empirical error on the importance-weighted examples. Define the random examples for the empirical errors:
\begin{align*}
    E_i:= \left( \frac{\one[h(x_i) \neq y_i] -\one[h(x_i) \neq y_i^\WL] }{\dot{P}_i(x_i)} Q_i + \one[h(x_i) \neq y_i^\WL] \right)\one[x_i \in D_i]
\end{align*}

and the associated $\sigma$-fields $\calF_i:=\sigma\left(\left\{X_j, Y_j, Q_j\right\}_{j=1}^i\right)$. By the same analysis of the sequence of instantaneous regrets, we have $E_i-\E\left[E_i \mid \calF_{i-1}\right]$ is a martingale difference sequence adapted to the filtrations $\calF_i, i \geq 1$, with the following properties:
\begin{align*}
    &\E\left[E_i \mid \calF_{i-1}\right] =\E\left[\one\left(X_i \in D_i \wedge h\left(X_i\right) \neq Y_i\right) \mid \calF_{i-1}\right]=\operatorname{err}_{m(i)}(h)\\
    &\left|E_i-\E\left[E_i \mid \calF_{i-1}\right]\right| \leq \frac{1}{P_{\min , m(i)}} \leq \frac{1}{P_{\min , m}} \text{ for all } m \geq m(i) \\
    &  \E\left[\left(E_i - \E[E_i | \calF_{i-1}]\right)^2 |\calF_{i-1}\right] \leq \errvar_\WL(h,m(i))
\end{align*}
Again, by applying the variant of freedman inequality and other subtle steps in the original paper, we get the desired bound.
 
\end{proof}

\subsection{Concentrations under the WL leveraging strategy}
Let $\calE_\text{train}$ denote the event that the assertions of Lemma~\ref{lem: deviation bound} and we know that $\Pr[\calE_\text{train}]\leq 1-\delta$, we obtain the following propositions for the concentration of empirical regret and error terms. As we shown below, our results is very similar to the original one. We provides some key proof steps to illustrate in details.

\begin{proposition}[Regret concentration -- modified prop 1]
\label{prop: regret concentration}
Fix an epoch $m \geq 1$, suppose the events $\calE_\WL$ and  $\calE_\text{train}$ holds and assume $h^* \in A_j$ for all epoch $j \leq m$,
\begin{align*}
    &\left|\reg\left(h, h',Z_{m}\right)-\tilreg_{m}(h,h')\right|\\
    &\leq \frac{1}{4}\tilreg_m(h) + 2\alpha\sqrt{\frac{\epsilon_m}{\tau_m}\sum_{j=1}^m (\tau_j - \tau_{j-1})\reg_j(h)}
     + 2\alpha\sqrt{3\berr_m(h^*) \epsilon_m} \\
        &\quad + \beta \sqrt{2\gamma\epsilon_m \Delta_m \sum_{j=1}^m  (\tau_j - \tau_{j-1}) (\reg(h,\tilde{Z}_{j-1})+\reg(h^*,\tilde{Z}_{j-1})} + 5\Delta_m
\end{align*}
{\color{blue} To compare, the original bound is almost same except the slightly different constant requirements on $c_1,c_2$ inside the definition of $\Delta_m$
}
\end{proposition}
\begin{proof}
One key step in the original proof is to show that both $\regvar_\WL(h,h^*,j)$ and $\regvar_\NWL(h,h^*,j)$ for any fixed $j$ is upper bounded by the follows (see the second equation block in \textbf{Proof of Proposition 1 in original paper}).
\begin{align*}
    &\E\large[2 \alpha^{2} \one\left(x \in D_{j}\right)\left(\one\left(h(x) \neq h_i(x)\right)+\one\left(h^*(x) \neq h_i(x)\right)\right)\\
        & \quad +2 \beta^{2} \gamma \tau_{j-1} \Delta_{j-1}\left(\reg\left(h, \tilde{Z}_{j-1}\right)+\reg\left(h^*, \tilde{Z}_{j-1}\right)\right)+2 \xi \tau_{j-1} \Delta_{j-1}^{2} \\
        & \quad +\one\left(h(x) \neq h^*(x) \wedge x \notin D_{j}\right) \large]
\end{align*}
Note the this upper bound for $\regvar_\NWL(h,h^*,j)$ has been proved in the original paper, and here we will focus on showing the upper bound of $\regvar_\WL(h,h^*,j)$. Firstly, we have
\begin{align*}
     & \E\left[\left(\left(\frac{2}{\Dot{P}_j(x)} \one[y^\WL \neq y] + 1\right)\one[x \in D_j] + \one[x_i \notin D_j]\right)\one[h(x) \neq h^*(x)] \right] \\
     & \leq \E\left[\left(\frac{2}{\Dot{P}_j(x)} \one[y^\WL \neq y] + 1\right)\one[x \in D_j]  \left(\one[h(x) \neq h_m(x)] + \one[h^*(x) \neq h_m(x)]\right)  + \one[x \notin D_j] \one[h(x) \neq h^*(x)] \right]
\end{align*}
Next we show that the output of \textbf{OP}~\ref{algo: op} can yield the target upper bound on the first term. That is, for any $h$, we have
\begin{align*}
    &\E\left[2\frac{\one\left(h(x) \neq h_j(x) \wedge x \in D_{j} \wedge y^\WL \neq y\right)}{\Dot{P}(x)}
        + \one[h(x) \neq h_j(x) \wedge x \in D_{j}] \right] \\
    & = \E\left[2\frac{\one\left(h(x) \neq h_j(x) \wedge x \in D_{j}\right) \E[\one[y^\WL \neq y | x]]}{\Dot{P}(x)} \right]
    + \E\left[\one[h(x) \neq h_j(x) \wedge x \in D_{j}] \right] \\
    & \leq \E\left[2\frac{\one\left(h(x) \neq h_j(x) \wedge x \in D_{j}\right) }{\Dot{P}(x)} \right] \hinerr(D_j)
    + \E\left[\one[h(x) \neq h_j(x) \wedge x \in D_{j}] \right] \\
    & \leq \E\left[2\frac{\one\left(h(x) \neq h_j(x) \wedge x \in D_{j}\right) }{\Dot{P}(x)} \right] \Dot{\hinerr}_j
    + \E\left[\one[h(x) \neq h_j(x) \wedge x \in D_{j}] \right] \\
    & \leq b_j^\NWL(h)
\end{align*}
where the first inequality comes from Assumption~\ref{ass: hint oracle}, the second inequality comes from Lemma~\ref{lem: conditional hint error est} and the last inequality comes directly from optimization constrains in \textbf{OP}~\ref{algo: op}.

By choosing the $h = h \text{ or } h^*$ and combine these two, we have 
\begin{align*}
     \E\left[\left(\left(\frac{2}{\Dot{P}_j(x)} \one[y^\WL \neq y] + 1\right)\one[x \in D_j] + \one[x_i \notin D_j]\right)\one[h(x) \neq h^*(x)] \right]
     \leq 2b_j^\NWL(h) + \one[x \notin D_j] \one[h(x) \neq h^*(x)],
\end{align*}
which is our desired results by replacing our definition of $b_j^\NWL$ . Then all the following proofs are as the same as the original proofs and we will skip here. In the end we have
\begin{align*}
    &\left|\reg\left(f, f',Z_{m}\right)-\tilreg_{m}(f,f')\right|\\
    &\leq \frac{1}{4}\tilreg_m(h) + 2\alpha^2\epsilon_m + 2\alpha\sqrt{\frac{\epsilon_m}{\tau_m}\sum_{j=1}^m (\tau_j - \tau_{j-1})\reg_j(h)}
     + 2\alpha\sqrt{3\berr_m(h^*) \epsilon_m} \\
        &\quad + \beta \sqrt{2\gamma\epsilon_m \Delta_m \sum_{j=1}^m  (\tau_j - \tau_{j-1}) (\reg(h,\tilde{Z}_{j-1}+\reg(h^*,\tilde{Z}_{j-1})} + \Delta_m + \frac{2\epsilon_m}{P_{\min,m}}\\
    &\leq \frac{1}{4}\tilreg_m(h) + 2\alpha\sqrt{\frac{\epsilon_m}{\tau_m}\sum_{j=1}^m (\tau_j - \tau_{j-1})\reg_j(h)}
     + 2\alpha\sqrt{3\berr_m(h^*) \epsilon_m} \\
        &\quad + \beta \sqrt{2\gamma\epsilon_m \Delta_m \sum_{j=1}^m  (\tau_j - \tau_{j-1}) (\reg(h,\tilde{Z}_{j-1}+\reg(h^*,\tilde{Z}_{j-1})} + 5\Delta_m
\end{align*}
\end{proof}

\begin{proposition}[Error concentration -- modified prop 2]
\label{prop: error concentration}
Fix an epoch $m \geq 1$, suppose the events $\calE_\WL$ and  $\calE_\text{train}$ holds and assume $h^* \in A_j$ for all epoch $j \leq m$,
\begin{align*}
\left|\overline{\err}_{m}\left(h^*\right)-\err\left(h_{m+1}, \tilde{Z}_{m}\right)\right|
\leq  (\frac{5}{2} + 4\log \tau_m )  \Delta_m + \frac{\berr_m(h^*)}{2} + \frac{\err(h_{m+1},\tilde{Z}_m)}{2}
        +\reg\left(h^*, h_{m+1}, \tilde{Z}_{m}\right).
\end{align*}
{ \color{blue} To compare, the original bound without WL is
\begin{align*}
\left|\overline{\err}_{m}\left(h^*\right)-\err\left(h_{m+1}, \tilde{Z}_{m}\right)\right| 
\leq \frac{\berr_{m}\left(h^*\right)}{2}+\frac{3 \Delta_{m}}{2}
        +\reg\left(h^*, h_{m+1}, \tilde{Z}_{m}\right).
\end{align*}
Later we will show that $\err(h_{m+1},\tilde{Z}_m)$ in our result can be combined with other terms. So our upper bound is slightly enlarged compared to previous one but still in the same order. 
}
\end{proposition}

\begin{proof}
Again we follow the similar proof steps as in the \text{Proof of Proposition}. Thus, based on Lemma~\ref{lem: deviation bound}, the key step is to show the following upper bounds
\begin{align*}
    & \sqrt{ \frac{\epsilon_{m}}{\tau_{m}} \left(\sum_{ j \in M_{m,\WL}} (\tau_{j}-\tau_{j-1}) \errvar_\WL(h,j) + \sum_{ j \in M_{m,\NWL}} (\tau_{j}-\tau_{j-1})  \errvar_\NWL(h,j)\right)}  + \frac{\epsilon_{m}}{P_{\min , m}} \\
    & \leq (\frac{5}{2} + 4\log \tau_m) \Delta_m + \frac{\berr_m(h^*)}{2} + \frac{\err(h_{m+1},\tilde{Z}_m)}{2}
\end{align*}
First of all, by using the same approach as in the original proofs, we can easily upper bound the \NWL~term
$\frac{\epsilon_{m}}{\tau_{m}} \left(\sum_{ j \in M_{m,\NWL}} (\tau_{j}-\tau_{j-1})  \errvar_\NWL(h,j)\right)$ by 
$
    \frac{\epsilon_m \berr_m(h^*)}{P_{m,\min}}.
$

Now we will focus on upper bound the \WL term as follows,
\begin{align*}
    & \frac{\epsilon_{m}}{\tau_{m}}\sum_{ j \in M_{m,\WL}} (\tau_{j}-\tau_{j-1})  \errvar_\WL(h,j) \\
    & \leq   \frac{\epsilon_{m}}{\tau_{m}} \sum_{ j \in M_{m,\WL}} (\tau_{j}-\tau_{j-1}) \E\left[\frac{\one\left(x \in D_{j} \wedge y^\WL \neq y\right)}{P_{\min , m}}\right] \\
    & \leq \frac{\epsilon_{m}}{\tau_{m}} \sum_{ j \in M_{m,\WL}} (\tau_{j}-\tau_{j-1})\frac{ \err(h_j,\tilde{Z}_{j-1}) + \Delta_{j-1}}{P_{m,\min}}\\
    & \leq \frac{\epsilon_{m}}{\tau_{m}} \sum_{ j \in M_{m,\WL}} \frac{\tau_{j}-\tau_{j-1}}{\tau_{j-1}}  \tau_m \frac{ \err(h_{m+1},\tilde{Z}_{m})}{P_{m,\min}} +\frac{\epsilon_{m}}{\tau_{m}} \sum_{ j \in M_{m,\WL}} (\tau_{j}-\tau_{j-1})\frac{ \Delta_{j-1}}{P_{m,\min}} \\
    & \leq  \frac{\epsilon_m\log \tau_m}{P_{m,\min}} (\err(h_{m+1},\tilde{Z}_{m}) + \Delta_m)
\end{align*}
where the second inequality comes from the optimism of $\overline{\hinerr}_j$ when $\calE_\WL$ holds and the condition $\overline{\hinerr}_k \leq \berr(h_{j}, \tilde{Z}_{j-1})$ in Step 2 of Algo.~\ref{algo: hint-eval}, the third inequality comes from Lemma~\ref{lem: helper 1} and the last inequality comes from Lemma 8 in the original paper.
Combine these two terms, we have the target bound
\begin{align*}
    \left|\err\left(h^*, Z_{m}\right)-\overline{\err}_{m}(h^*)\right|
    & \leq  \sqrt{\frac{\epsilon_m \berr_m(h^*)}{P_{m,\min}} 
        +  2 \log \tau_m \frac{ \epsilon_m\err(h_{m+1},\tilde{Z}_{m})}{P_{m,\min}}
        + 2 \log \tau_m \frac{ \epsilon_m\Delta_m}{P_{m,\min}}} 
        + \frac{\epsilon_{m}}{P_{\min , m}} \\ 
    & \leq (\frac{3}{2} + 4\log \tau_m) \frac{\epsilon_m}{P_{m,\min}} + \frac{\berr_m(h^*)}{2} + \frac{\err(h_{m+1},\tilde{Z}_m)}{2}
        + \Delta_m \\
    & \leq (\frac{5}{2} + 4\log \tau_m) \Delta_m + \frac{\berr_m(h^*)}{2} + \frac{\err(h_{m+1},\tilde{Z}_m)}{2}.
\end{align*}
Therefore, by using the same observation as in the original paper, we have
\begin{align*}
    \left|\overline{\err}_{m}\left(h^*\right)-\err\left(h_{m+1}, \tilde{Z}_{m}\right)\right| 
    &\leq \left|\err\left(h^*, Z_{m}\right)-\overline{\err}_{m}(h^*)\right| + \reg(h^*, h_{m+1}, \tilde{Z}_m)\\
    &\leq \text{target bound}
\end{align*}
\end{proof}

\begin{corollary}
\label{coro: error concentration}
Based on  proposition \ref{prop: error concentration}, we can get the following error estimation guarantees,
\begin{align*}
    & \berr_{m}\left(h^*\right)  
    \leq (5 + 8\log \tau_m) \Delta_{m} + 3\err(h_{m+1},\tilde{Z}_m) + 2\reg\left(h^*, h_{m+1}, \tilde{Z}_{m}\right)\\
    & \err(h_{m+1},\tilde{Z}_m)
    \leq (5 + 8\log \tau_m) \Delta_{m} + 3\berr_{m}\left(h^*\right) + 2\reg\left(h^*, h_{m+1}, \tilde{Z}_{m}\right)\\
\end{align*}
\end{corollary}

\subsection{Main results and its analysis}

Based on these two propositions, we are ready to prove this general version of the theorem.

\begin{theorem}
\label{them: general}
For all epochs $m = 1,2,\ldots,M$ and all $h \in \calH$, the following holds with probability at least $1-2\delta$,
\begin{align*}
    & |\reg(h,h^*,\tilde{Z}_m) - \tilreg_m(h,h^*)| \leq \frac{1}{2}\tilreg_m(h,h^*) + \frac{\eta}{4} \Delta_m\\
    & \reg(h^*,h_{m+1},\tilde{Z}_m) \leq \frac{\eta \Delta_m}{4} \text{ and } h^*\in A_m\\
    &  \err(h_{m+1},\tilde{Z}_m)
    \leq (5 + 8\log \sum_{j=1}^{m} L_j+\frac{\eta}{2}) \Delta_{m} + 3\berr_{m}\left(h^*\right) \\
    &  \berr_{m}\left(h^*\right) 
    \leq (5 + 8\log \sum_{j=1}^{m} L_j+\frac{\eta}{2}) \Delta_{m} + 3\err(f_{m+1},\tilde{Z}_m)
\end{align*}
\end{theorem}
\begin{proof}
First of all, assume $\calE_\WL$ and $\calE_\text{train}$ holds, so all the previous results holds.  

Now this theorem is proved inductively and follows the similar step as in Section 7.2.1 in the original proof. Firstly because our proposition~\ref{prop: regret concentration} gives the same result as in original proof, so we have for any block $m > 1$,
\begin{align*}
    &\left|\reg\left(h, h^*,Z_{m}\right)-\tilreg_{m}(h,h^*)\right|\\
    &\leq \frac{1}{4}\tilreg_m(f) + \underbrace{2\alpha\sqrt{\frac{\epsilon_m}{\tau_m}\sum_{j=1}^m (\tau_j - \tau_{j-1})\reg_j(f_j)}}_{\calT_1}
     +  \underbrace{2\alpha\sqrt{3\berr_m(h^*) \epsilon_m}}_{\calT_2} \\
        &\quad + \underbrace{\beta \sqrt{2\gamma\epsilon_m \Delta_m \sum_{j=1}^m  (\tau_j - \tau_{j-1}) (\reg(f,\tilde{Z}_{j-1}+\reg(h^*,\tilde{Z}_{j-1})}}_{\calT_3} + 5\Delta_m.
\end{align*}

Now we are ready to upper bound these three terms. In order to prove inductively, we use $\calE_{\text{train},m}$ to state that the upper bounds in theorem hold for block $m$.  We can easily bound $\calT_1$ and $\calT_3$ using the fact $\calE_{\text{train},m-1}$ holds, the analysis for this part is exactly the same as original proof so we will skip the details and directly stated the result. 
\begin{align*}
    &\calT_1 \leq \frac{\eta \Delta_m}{12}+24 \alpha^2 \epsilon_m \log \tau_m \\
    & \calT_3 \leq \frac{1}{4} \widetilde{\mathrm{reg}}_m\left(h, h^*\right)+\frac{7 \eta \Delta_m}{72}
\end{align*}
So here our main focus is to show the upper bound of $\calT_2$ which is effected by our modified algorithm. By applying  the first inequality in Corollary~\ref{coro: error concentration},  we get simplify $\calT_2$ as
\begin{align*}
    \calT_2
    &= 2\alpha \sqrt{3\epsilon_m \berr_m(h^*)}
    \leq 2\alpha \sqrt{ 3 \epsilon_m \left(  (5 + 8\log \tau_m) \Delta_{m} + 3\err(h_{m+1},\tilde{Z}_m) + 2\reg\left(h^{*}, h_{m+1}, \tilde{Z}_{m}\right)\right)} \\
    & \leq 2\alpha \sqrt{15\epsilon_m \err(h_{m+1},\tilde{Z}_m)}
        + 2\alpha \sqrt{(15 + 24\log \tau_m)\epsilon_m \Delta_m }
        + 2\alpha \sqrt{6\epsilon_m \reg\left(h^{*}, h_{m+1}, \tilde{Z}_{m}\right)}\\
    & \leq  2\alpha \sqrt{15\epsilon_m \err(h_{m+1},\tilde{Z}_m)}
            + \Delta_m
            + (30+48 \log\tau_m) \alpha^2\epsilon_m
            + \frac{1}{4}\reg(h^*,h_{m+1},\tilde{Z}_m)
\end{align*}
Therefore combine the upper bound of $\calT_2$ with bounds of $\calT_1,\calT_3$, we have
\begin{align*}
    &\left|\reg\left(h, h^*,Z_{m}\right)-\tilreg_{m}(h,h^*)\right|\\
    &\leq \frac{1}{4}\tilreg_m(h) + \frac{\eta \Delta_m }{12} + 24\alpha^2\epsilon_m \log\tau_m \\
        & \quad  + 2\alpha\sqrt{9\epsilon_m \err(h_{m+1},\tilde{Z}_m)} + \Delta_m + \frac{1}{4}\reg(h^*,h_{m+1},\tilde{Z}_m) +  (30+48 \log\tau_m)\alpha^2\epsilon_m\\
        & \quad  + \frac{1}{4}\tilreg_m(h,h^*)
        + \frac{7\eta \Delta_m}{72}+ 5\Delta_m \\
    & \leq \frac{1}{2}\tilreg_m(f) + 102\alpha^2\epsilon_m \log \tau_m  +  \frac{13\eta \Delta_m }{72}
        + 2\alpha\sqrt{9\epsilon_m \err(h_{m+1},\tilde{Z}_m)} + 6\Delta_m  + \frac{1}{4}\reg(h^*,h_{m+1},\tilde{Z}_m)
\end{align*}
Further recalling that $c_1 \geq 6\alpha$ and $c_2 \geq 102\alpha^2$ by our assumptions on constants, we obtain
\begin{align*}
    \left|\reg\left(h, h^*,Z_{m}\right)-\tilreg_{m}(h,h^*)\right|
    \leq \frac{1}{2}\tilreg_m(h) +  \frac{13\eta \Delta_m }{72}
        + 7\Delta_m  + \frac{1}{4}\reg(h^*,h_{m+1},\tilde{Z}_m)
\end{align*}
To complete the proof of the bound (36), we now substitute $h=h_{m+1}$ in the above bound, which yields
$$
\frac{1}{2} \widetilde{\widetilde{\mathrm{reg}}_m}\left(h_{m+1}, h^*\right)-\frac{5}{4} \operatorname{reg}\left(h, h^*, \tilde{Z}_m\right) \leq \frac{13 \eta}{72} \Delta_m+7 \Delta_m
$$
Since $h^* \in A_i$ for all epochs $i \leq m$, we have $\widetilde{\text { reg }}\left(h, h^*\right) \geq \operatorname{reg}\left(h, h^*\right) \geq 0$ for all classifiers $h \in \mathcal{H}$. Consequently, we see that
$$
\operatorname{reg}\left(h^*, h_{m+1}, \tilde{Z}_m\right)=-\operatorname{reg}\left(h_{m+1}, h^*, \tilde{Z}_m\right) \leq \frac{52 \eta}{360} \Delta_m+\frac{28}{5} \Delta_m \leq \frac{\eta}{4} \Delta_m
$$
\\\\
Finally, from Corollary~\ref{coro: error concentration}and the previous upper bound result on $\reg(h^*,h_{m+1},\tilde{Z}_m)$, we can get the third and the forth inequality. That is,

\begin{align*}
    \berr_{m}\left(h^*\right)  
    &\leq (5 + 8\log \tau_m) \Delta_{m} + 3\err(h_{m+1},\tilde{Z}_m) + 2\reg\left(h^*, h_{m+1}, \tilde{Z}_{m}\right) \\
    & \leq (5 + 8\log \tau_m) \Delta_{m} + 3\err(h_{m+1},\tilde{Z}_m) + \frac{\eta \Delta_m}{2}\\
    \err(h_{m+1},\tilde{Z}_m)
    & \leq (5 + 8\log \tau_m) \Delta_{m} + 3\berr_{m}\left(h^*\right) + 2\reg\left(h^*, h_{m+1}, \tilde{Z}_{m}\right)\\
    & \leq  (5 + 8\log \tau_m) \Delta_{m} + 3\berr_{m}\left(h^*\right) + \frac{\eta \Delta_m}{2}
\end{align*}

Recall that we simplify the notation in proof and the $\tau_m$ is actually $\sum_{j=1}^m L_j$. So we get the target bound.
\end{proof}

Now we proof the Theorem~\ref{them: main} as a direct follow-up.
\begin{theorem}[Main theorem, Restate]
Pick any $0<\delta<1 /$ e such that $|\mathcal{H}| / \delta>\sqrt{192}$. Then recalling that $h^*=\arg \min _{h \in \mathcal{H}} \operatorname{err}(h)$, we have for all epochs $m=1,2, \ldots, M$, with probability at least $1-\delta$
\begin{align*}
    \reg\left(h, h^*\right)\leq \order( \Delta_m^*) \quad \text { for all } h \in A_{m+1}
\end{align*}
where $\Delta_0^* =\Delta_0 \text { and } \Delta_m^*:=c_1 \sqrt{\epsilon_m \overline{\operatorname{err}}_m\left(h^*\right)}+c_2 \epsilon_m \log(\sum_{j=1}^{m} L_j) \text { for } m \geq 1 $
\end{theorem}

\begin{proof}
By using the exact same proof in the original paper we have easily get 
\begin{align*}
    \reg(h) \leq 4\gamma \Delta_m.
\end{align*}
 We skip the proof here. Now we show our modified version can again leads to $\Delta_m \leq 4\Delta_m^*$. It is trivial true for m = 1 because $\Delta_1^* = \Delta_1$. For $m \geq 2$, we have 
\begin{align*}
    \Delta_m 
    &\leq c_1\sqrt{\epsilon_m\left( \err(h_{m+1},\tilde{Z}_m) \right)} + c_2\epsilon_m \log\tau_m\\
    & \leq  c_1\sqrt{\epsilon_m\left((5 + 8\log \tau_m) \Delta_{m} + 3\berr_{m}\left(h^{*}\right) + 2\reg\left(h^{*}, h_{m+1}, \tilde{Z}_{m}\right) \right)} + c_2\epsilon_m \log\tau_m\\
    & \leq c_1\sqrt{\epsilon_m\left((\frac{\eta}{2}+13\log \tau_m )\Delta_{m} + 3\berr_{m}\left(h^{*}\right) \right)} + c_2\epsilon_m \log\tau_m\\
    &\leq c_1\sqrt{3\epsilon_m\berr_{m}\left(h^{*}\right)} + c_1\sqrt{\epsilon_m (\frac{\eta}{2}+13\log \tau_m )\Delta_{m}} + c_2\epsilon_m \log\tau_m\\
    & \leq c_1\sqrt{3\epsilon_m\berr_{m}\left(h^{*}\right)} + c_1^2\epsilon_m(\frac{\eta}{2}+13\log \tau_m ) + \frac{\Delta_m}{2} + c_2\epsilon_m \log\tau_m \\
    & \leq 2\Delta_m^* + \frac{\Delta_m}{2}
\end{align*}
where the last inequality uses our choice of constants $c_1^2(\frac{\eta}{2}+13 ) \leq c_2$. Rearrange terms and again recall that we simplify the notation in proof and the $\tau_m$ is actually $\sum_{j=1}^m L_j$, so we complete the proof. 
\end{proof}

\subsection{Auxiliary lemma}
\begin{lemma}
\label{lem: helper 1}
 For any fixed block number $m$, we have
 \begin{align*}
     \sum_{i=1}^{m} (\tau_{i}-\tau_{i-1}) \err(h_i, \tilde{Z}_{i-1})
    \leq  4 \tau_{m} \log \tau_{m+1} \err(h_{m+1}, \tilde{Z}_{m} )
 \end{align*}
\end{lemma}
\begin{proof}
It is easy to see that, for any $m \geq i$
\begin{align*}
     \tau_{i-1}  \err(h_i, \tilde{Z}_{i-1})
     &\leq  \tau_{i-1}  \err(h_{i+1}, \tilde{Z}_{i-1})\\
     &\leq \tau_i \err(h_{i+1}, \tilde{Z}_{i})\\
     &\leq \ldots
     \leq  \tau_{m-1}  \err(f_m, \tilde{Z}_{m-1})
\end{align*}
Now by the fact
$
    \sum_{i=1}^{m} \frac{\tau_{i+1}-\tau_{i}}{\tau_{i}} \leq 4 \log \tau_{m+1}
$
in \cite{https://doi.org/10.48550/arxiv.1506.08669}, Lemma 8, we have
\begin{align*}
     \sum_{i=1}^{m} (\tau_{i}-\tau_{i-1}) \err(h_i, \tilde{Z}_{i-1})
     & = \sum_{i=1}^{m}  \frac{(\tau_{i}-\tau_{i-1})}{\tau_{i-1}}  \tau_{i-1}  \err(h_i, \tilde{Z}_{i-1}) \\
     & \leq \sum_{i=1}^{m} \frac{(\tau_{i}-\tau_{i-1})}{\tau_{i-1}} \tau_{m}  \err(h_{m+1}, \tilde{Z}_{m} )\\
     & \leq  4 \tau_{m} \log \tau_{m+1} \err(h_{m+1}, \tilde{Z}_{m} )
\end{align*}
\end{proof}

\section{Appendix: Label complexity and its analysis}
\label{sec: label complexity result (supp)}

\subsection{Analysis for OP~\ref{algo: op}}
The label complexity in Phase 1 and Phase 2 in block $m$ can be explicitly written as follows
\begin{lemma}
\label{lem: label complexity form}
    When $\calE_\WL$ holds, for any fixed block $m$ with $\textsc{Use-WL}= True$, 
    \begin{align*}
        &
        \E\left[\sum_{t =\tau_{m-1}+1}^{\Dot{\tau}_m} \one[ y_t \text{ is queried}] \right] \leq (\Dot{\tau}_m - \tau_{m-1}) \E[x \in D_m]\\
        & 
        \E\left[\sum_{t =\Dot{\tau}_m}^{\tau_{m}} \one[ y_t \text{ is queried}] \right] \leq \max\{ 4N_m\Dot{\hinerr}_m, L_m P_{m,\min}\E[\one[x \in D_m]]\}.
    \end{align*}
\end{lemma}

\begin{proof}
    The first one comes from the fact that we only query samples inside $D_m$ in the Phase 1.
    For the second result, notice that by choosing $\Dot{P}_m(x) = {4 P_m(x)} \Dot{\hinerr}_m$, we get a feasible solution of \textbf{OP} \ref{algo: op} without considering the $P_{m,\min}$, as shown below
    \begin{align*}
        & \E_{x}\left[\frac{\one\left(h(x) \neq h_m(x) \wedge x \in D_{m}\right)}{\Dot{P}_m(x)}\right] \Dot{\hinerr}_m \\
        & = \E_{x}\left[\frac{\one\left(h(x) \neq h_m(x) \wedge x \in D_{m}\right)}{4 P_m(x) \Dot{\hinerr}_m} \right] \Dot{\hinerr}_m \\
        & = \frac{1}{4}\E_{X}\left[\frac{\one\left(h(x) \neq h_m(x) \wedge x \in D_{m}\right)}{P_m(X)}\right] \\
        & \leq \frac{1}{4}b_{m}^{\NWL}(h)
        \leq \frac{1}{2}(b_{m}^{\NWL}(h) - \E_x[\one_h^m(x)])
    \end{align*}
    Now combine this with the minimum query probability requirement, we get the desired result.
\end{proof}

\subsection{Analysis for the unlabeled sample complexity}
First we show that, the number of unlabeled sample we draw from Phase1 is at most the same order as the number in Phase 2, which means out results in terms of $n$ remains in the same order as before.

\begin{lemma}
\label{lem: Unlabeled sample guarantee}
For any fixed block $m$, $\Dot{\tau}_m - \tau_{m-1} \leq 8L_m$.
\end{lemma}
\begin{proof}
According the stopping condition of the algorithm, at the end of the block $m$, we have the number of new drawn unlabeled samples as $2^{K_m^e+1}\ln(2M\log(n)/\delta)$. When $K_m^e = K_m^s$, then we directly have
\begin{align*}
    2^{K_m^e+1}\ln(2M\log(n)/\delta)
    \leq \frac{12*2\ln(2M\log(n)/\delta)}{\berr(h_m, \Tilde{Z}_{m-1})}
    \leq \frac{N_m}{\E[\one[x \in D_m]]} 
    \leq \frac{L_m\E[\one[x \in D_m]]}{\E[\one[x \in D_m]]} 
    =L_m
\end{align*}
When $K_m^e > K_m^s$, 
\begin{align*}
    2^{K_m^e+1}\ln(2M\log(n)/\delta)
    &= 2 * 2^{(K_m^e-1)+1}\ln(2M\log(n)/\delta) \\
    &\leq 2* \frac{\max\{4N_m \frac{\overline{\hinerr}_{K_m^e-1}}{\E[\one[x \in D_m]]}, L_m P_{m,\min}\E[\one[x \in D_m]]\}}{\E[\one[x \in D_m]]} \\
    & \leq 2*\frac{\max\{  4L_m \E[\one[x \in D_m]] , L_m P_{m,\min}\E[\one[x \in D_m]]\}}{\E[\one[x \in D_m]]}\\
    & \leq 8L_m
\end{align*}
\end{proof}

Therefore, roughly speaking if original AC requires $n$ number of unlabeled samples to achieve their regret guarantee, then here we at most need $\text{const}*n = \order(n)$ unlabeled sample to achieve the similar guarantees. This is weak requirement since unlabeled samples sources are usually unlimited.

\subsection{Analysis for Algo~\ref{algo: hint-eval} on the label complexity}

In the rest of this section, we show Algo~\ref{algo: hint-eval} can automatically balance Phase 1 and Phase 2.

\begin{lemma}
\label{lem: upper bound of k in phase 1}
    When $\calE_\WL$ hold, then for any fixed block $m$, if $K_m^e \geq K_m^s + 1$, 
    \begin{align*}
        K_m \leq  \max \left\{ \min \left\{\calT_1, \calT_2 \right\}, \calT_3\right\}
    \end{align*}
    where,
    \begin{align*}
        &\calT_1 = \log\left( 16\frac{\kappa_m  N_m \hinerr_m}{\E[x \in D_m]\ln(2M\log(n)/\delta)}  +\frac{\sqrt{6N_m \kappa_m}}{\E[x \in D_m]\sqrt{\ln(2M\log(n)/\delta)}} \right)\\
        &\calT_2 = \log\left(\frac{N_m}{\E[x \in D_m]\ln(2M\log(n)/\delta)} \right)\\
        &\calT_3 = \log\left(\frac{L_m P_{m,\min}}{\ln(2M\log(n)/\delta)} \right)\\
        &\hinerr_m : = \max_{x \in D_m} \E_y[\one[h(x) \neq y]|x]
    \end{align*}
\end{lemma}
\begin{proof}
    For $k = K_m^e-1$, by the stopping condition, we can upper bound the $\E[x \in D_m] 2^{K_m^e}\ln(2M\log(n)/\delta)$ as follows
    \begin{align*}
        &\max\{ 4N_m \min\left\{\frac{ \kappa_m \overline{\hinerr}_{m,k}}{\E[\one[x \in D_m]]},1\right\}, L_m\E[\one[x \in D_m]P_{m,\min}] \\
        &=  \max\{N_m\min\left\{ \frac{  \kappa_m\left( 2\wh{\hinerr}_k + 2^{-k}\right)}{\E[\one[x \in D_m]]} , 1 \right\}, L_m\E[\one[x \in D_m]P_{m,\min}\} \\
        & \leq  \max\{ N_m \min\left\{\kappa_m\left( 4 \hinerr_m + \frac{3*2^{-(K_m^e-1)}}{ \E[x \in D_m]}\right), 1 \right\}, L_m\E[\one[x \in D_m]]P_{m,\min}\}
    \end{align*}
    Solving this inequality give desired result.
\end{proof}

\begin{lemma}
\label{lem: main label complexity}
    When $\calE_\WL$ holds, for any fixed block $m$, we have the expected label complexity \#m upper bounded by
    \begin{align*}
         \begin{cases}
         N_m \leq \min \left\{N_m,  \left(\frac{24\ln(2 \log n /\delta)}{ \berr(h_m, \Tilde{Z}_{m-1})} \right) \E[\one[x \in D_m]]\right\}
         & \quad \text{when stop at Step 1}\\
         2N_m & \text{when stop at Step 2}
         \end{cases}
    \end{align*}

    On the other hand, when $\textsc{Use-WL} = \text{True}$, we have label complexity \#m upper bounded by
    \begin{align*}
         \begin{cases}
         & \max\left\{\min\left\{32\kappa_m  N_m \hinerr_m +\sqrt{24N_m \kappa_m\ln(2M\log(n)/\delta)},N_m \right\}, L_m P_{m,\min}\E[\one[x \in D_m]] \right\}
         \quad \text{when } K_m^e \geq  K_m^s + 1\\
         & \max\left\{\frac{24\ln(2 \log n /\delta)}{ \berr(h_m, \Tilde{Z}_{m-1})} ,0 \right\} \E[\one[x \in D_m]] 
         \leq N_m
         \quad \text{when } K_m^e = K_m^s
         \end{cases}
    \end{align*}
\end{lemma}

\begin{proof}
Suppose the \WL Evaluation algorithm stops at Step 1, then no samples are drawn in Phase 1 of block $m$. $\#m = N_m$. 

Suppose the \WL Evaluation algorithm stops at Step 2, by the previous condition\\
$N_m \leq \left(\frac{4\ln(2 \log n /\delta)}{\berr(h_m, \Tilde{Z}_{m-1})} \right) \E[\one[x \in D_m]]$, $\#m = 2N_m$

Finally we focus on the Step 3. By the stopping condition and Lemma~\ref{lem: label complexity form}, we always have 
\begin{align*}
    N_m \leq \E[ \one[x \in D_m]]2^{K_m^e+1}\ln(2M\log(n)/\delta)
\end{align*}
Now in the case that $K_m^e = K_m^s$, we directly have the upper bound
\begin{align*}
    \max\left\{\frac{24\ln(2 \log n /\delta)}{ \berr(h_m, \Tilde{Z}_{m-1})} - |\Dot{Z}_m|,0 \right\}\E[\one[x \in D_m]]
\end{align*}
Otherwise, by replace the upper bound of $K_m^e$ from Lemma~\ref{lem: upper bound of k in phase 1}, we finish the proof.
\end{proof}

Now before we going to the final proof, we need to upper bound the number of WL mode oracle in the next section.

\subsection{Analysis on number of hint mode blocks}

For each block $m$, there are two conditions that will lead to \textsc{No Hint} mode,
\begin{itemize}
    \item \textbf{Condition 1}: $N_m \leq \left(\frac{6}{\berr(h_m, \Tilde{Z}_{m-1})} \right) \E[\one[x \in D_m]]$. This suggests the biased estimated error for best hypothesis is too small that, even evaluating whether the hint performs better than this can cost too much samples in Phase 1.
    \item \textbf{Condition 2}: $\overline{\hinerr}_{m,K_m^s} \geq \err(h_m, \Tilde{Z}_{m-1})$. This suggested the estimated hint performance is worse than the biased estimated error for best hypothesis. therefore may deteriorate the active learning strategy.  
\end{itemize}

Firstly, we will show later in the proof of main theorem that the occurrence of Condition 1 is not important because it can be implicitly upper bounded in the end. So here we will only focus on condition 2.
\begin{lemma}
\label{lem: compare to best error}
    For any fixed block $m$, as long as $\E \one[x \in D_m \wedge y^\WL \neq y  ] \leq \frac{1}{8}\berr(h_m,\Tilde{Z}_{m-1})$, we have $\overline{\hinerr}_{K_m^s} \leq \berr(h_m, \Tilde{Z}_{m-1})$.
\end{lemma}
\begin{proof}
When $\calE_\WL$ holds, we have
\begin{align*}
    \overline{\hinerr}_{m,K_m^s}
    &\leq 4\E[\one[ y^\WL \neq y  \wedge x \in D_m] ] + 3*2^{-K_m^s}\\
    &\leq \frac{1}{2}\berr(h_m,\Tilde{Z}_{m-1}) + \frac{1}{2} \berr(h_m,\Tilde{Z}_{m-1})\\
    &= \berr(h_m,\Tilde{Z}_{m-1})
\end{align*}
where the first term of last inequality comes from the assumption and second term of the last inequality comes from our choice of $K_m^s$ in the algorithm.
\end{proof}

\subsection{Main results in Section~\ref{sec: label complexity result each block (main)} and the analysis}

\begin{theorem}[label complexity of block $m$ (Restate)]
Let    
\begin{align*}
    &\calT_{1,m} = N_m \kappa_m \hinerr_m\\
    &\calT_{2,m} = \sqrt{\ln(2M\log(n)/\delta)\kappa_m N_m} \\
        &\quad + \sqrt{\frac{L_m}{\tau_{m-1}} \log(|\calH|n/\delta) L_m\left(\log \sum_{j=1}^{m} L_j\right)\E[\one[x \in D_m]]\phi_m}\\
        &\quad + \ln(2M\log(n)/\delta) \left(\log \sum_{j=1}^{m} L_j\right)\phi_m  \\
    &\calT_{3,m} =  \one\left[\hinerr_m  >\frac{1}{\left(\log \sum_{j=1}^{m} L_j\right)\phi_m} \right]N_m 
\end{align*}
where $\hinerr_m = \hinerr(D_m)$.
So the expected label complexity within block $m$ is upper bound by
\begin{align*}
    \order\left(\min\left\{ \calT_{1,m} + \calT_{2,m} + \calT_{3,m}, N_m \right\} \right)
\end{align*}
\end{theorem}

\begin{proof}

Therefore, we can decompose the expected sample complexity within block $m$ as
\begin{align*}
    \E_{\tau_m}
    &\left[\sum_{t = \tau_{m-1}+1}^{\tau_m} \E_t\Big[\one[y_t \text{ is queried}]\Big]\right]\\
    & =  \underbrace{\one[K_m^e \geq K_m^s +1 ] \#m}_{\calT_1}
        + \underbrace{\one[K_m^e = K_m^s  \wedge \text{subalgo stop as step 3} ] \#m}_{\calT_2}\\
        & \quad + \underbrace{\one[\text{subalgo stop as step 1}]\#m}_{\calT_3}
        + \underbrace{\one[\one[\text{subalgo stop as step 2} | ]\#m}_{\calT_4}
\end{align*}
Now we are ready to separately bound these three terms by using Lemma~\ref{lem: main label complexity}.
\\\\
For $\calT_{1,m}$, we have,
\begin{align*}
    {\calT_1}
    &\leq  \one[K_m^e \geq K_m^s +1 ]  \\
        &\quad * \max\left\{\min\left\{32\kappa_m  N_m \hinerr_m +\sqrt{24N_m \kappa_m\ln(2M\log(n)/\delta)},N_m \right\}, L_m P_{m,\min}\E[\one[x \in D_m]] \right\}\\
    &\leq N_m\min\left\{32\kappa_m \hinerr_m,1 \right\}
        + 5 \sqrt{\ln(2M\log(n)/\delta)\kappa_m N_m}
        +  P_{m,\min }L_m\E[\one[x \in D_m]
\end{align*}
By the definition of $P_{m,\min }$, we can further upper bound the third term as 
\begin{align*}
      P_{m,\min }L_m \E[\one[x \in D_m]
      & \leq  \frac{L_m}{\sum_{j=1}^{m-1} L_j} \sum_{j=1}^{m-1} L_j\frac{c_{3}}{\sqrt{\frac{\sum_{j=1}^{m-1} L_j \berr\left(f_{m}, \tilde{Z}_{m-1}\right) }{n\epsilon_{M}}}} \E[\one[x \in D_m]\\
      & = c_3\frac{L_m}{\sum_{j=1}^{m-1} L_j} \sqrt{\frac{n\epsilon_M \sum_{j=1}^{m-1} L_j}{\berr\left(f_{m}, \tilde{Z}_{m-1}\right)}} \E[\one[x \in D_m]\\
      & \approx  c_3\frac{L_m}{\sum_{j=1}^{m-1} L_j} \sqrt{\frac{ \tau_M \epsilon_M \sum_{j=1}^{m-1} L_j}{\berr\left(f_{m}, \tilde{Z}_{m-1}\right)}} \E[\one[x \in D_m]\\
      & \approx \frac{L_m}{\sum_{j=1}^{m-1} L_j} \sqrt{\frac{\log(|\calH|n/\delta)\sum_{j=1}^{m-1} L_j}{\berr\left(f_{m}, \tilde{Z}_{m-1}\right)}} \E[\one[x \in D_m]\\
      & \approx \frac{L_m}{\sum_{j=1}^{m-1} L_j} \sqrt{\log(|\calH|n/\delta) \left(\sum_{j=1}^{m-1} L_j \right)\E[\one[x \in D_m]\frac{\E[\one[x \in D_m]}{\berr\left(f_{m}, \tilde{Z}_{m-1}\right)}} \\
      & \lessapprox  \sqrt{\frac{L_m}{\sum_{j=1}^{m-1} L_j}} \sqrt{\log(|\calH|n/\delta)\log \left(\sum_{j=1}^{m} L_j\right) L_m \E[\one[x \in D_m]\phi_m}
\end{align*}
where the first inequality comes from Lemma~\ref{lem: upper bound of P_min} and the lat inequality comes from Lemma~\ref{lem: upper bound of 1/berr}.
\\
For $\calT_{2,m}$ and $\calT_{3,m}$,  we can combine them as
\begin{align*}
    \calT_2+ \calT_3
    & \leq  \min\left\{\E[x \in D_m]\left[\frac{12\ln(2M\log(n)/\delta)}{\berr(h_m,\Tilde{Z}_{m-1})} \right],N_m \right\} \\
    & =  \min\left\{12\ln(2M\log(n)/\delta)\log(\sum_{j=1}^{m} L_j) \phi_m,N_m \right\}.
\end{align*}
Again the last inequality comes form Lemma~\ref{lem: upper bound of 1/berr}.
\\
Finally for $\calT_{4,m}$, by using Lemma~\ref{lem: compare to best error}, we have
\begin{align*}
    \calT_4
    & \leq \one\left[\overline{\hinerr}_{K_m^e} > \berr(h_m, \Tilde{Z}_{m-1})\right]N_m \\
    & \leq  \one\left[\E\left[ \one[h(x) \neq y \wedge x \in D_m]\right] > \berr(h_m, \Tilde{Z}_{m-1})\right]N_m \\
    & \leq \one\left[\hinerr_m \E\left[ \one[x \in D_m]\right] > \err(h_m, \Tilde{Z}_{m-1}) + \Delta_m\right]N_m \\
    & = \one\left[\hinerr_m  >\frac{1}{\log(\sum_{j=1}^{m} L_j) \phi_m} \right]N_m 
\end{align*}
\end{proof}

\begin{theorem}[Upper bound of $\phi_m$]
For any block $m$, with probability at least $1-\delta$, we have always upper bound $ \phi_m \leq \min\{\theta^*, \frac{1}{\berr_m(h^*) + \log \left(\sum_{j=1}^{m} L_j\right) \epsilon_m}\}$. 
\end{theorem}
\begin{proof}
When $\calE_\text{train}$ holds, we can upper bound the $\E\left[ \one[x \in D_m]\right] $ as
\begin{align*}
    \E\left[ \one[x \in D_m]\right] 
    &\leq  \theta^* \left(16\gamma \Delta_{m-1}^* + 2\err_m(h^*) \right)\\
    & \leq  \theta^* \left(16\gamma \Delta_{m-1}^* + 2 \overline{\err}_m(h^*)\right) \\
    & =  \theta^* \left(16\gamma \left(c_1 \sqrt{\epsilon_{m-1} \overline{\err}_{m-1} (h^*))} + c_2 \epsilon_m \log \left(\sum_{j=1}^{m} L_j\right) \right)+ 2 \overline{\err}_m(h^*)\right)\\
    & \leq  \theta^* \left(16\gamma(2c_1+ c_2\log \left(\sum_{j=1}^{m} L_j\right)) \epsilon_m +  (8\gamma c_1 + 2)  \overline{\err}_m(h^*)\right)
\end{align*}
where the first inequality comes from eqn.(62) in the original paper. 
Therefore, by definition of $\phi_m$ we directly have
\begin{align*}
         \frac{\E\left[ \one[x \in D_m]\right]}{\berr_m(h^*) + \log \left(\sum_{j=1}^{m} L_j\right) \epsilon_m} 
         \leq \order(\theta^*)
\end{align*}
Alternatively, we can upper bound it as 
\begin{align*}
         \frac{\E\left[ \one[x \in D_m]\right]}{\berr_m(h^*) + \log \left(\sum_{j=1}^{m} L_j\right) \epsilon_m} 
         \leq \frac{1}{\berr_m(h^*) + \log \left(\sum_{j=1}^{m} L_j\right) \epsilon_m}
\end{align*}
\end{proof}

\subsection{Main results for Section~\ref{sec: label complexity result total (main)} and its analysis}

\begin{theorem}[Total label complexity under benign setting]
Suppose we end up dividing $n$ incoming labels into $M$ block, then with probability at least $1-\delta$, we have the total label complexity upper bounded by 
\begin{align*}
    \min\left\{\calT_1^\text{total} + \calT_2^\text{total} + \calT_3^\text{total}, N_\text{total}\right\},
\end{align*}
where,
\begin{align*}
    & \calT_1^\text{total} = \sum_{m=1}^M N_m \max\{\kappa_m \hinerr_m,1\} 
    \\& \calT_2^\text{total} =  \sqrt{\ln(2M\log(n)/\delta) \left(\sum_{m=1}^M \max\{\kappa_m,\frac{N_m}{\ln(2M\log(n)/\delta)}\}\right) N_\text{total} } \\
        & \quad + \sqrt{\log(|\calH|n/\delta)\left(\log \sum_{j=1}^{m} L_j\right) \sum_{m=1}^M L_m\E[\one[ x \in D_m]] \sum_{m=1}^M \frac{L_m}{\sum_{j=1}^{m-1}L_j}\phi_m }\\
        &\quad+ \sum_{m=1}^M \ln(2M\log(n)/\delta) \left(\log \sum_{j=1}^{m} L_j\right)\phi_m 
    \\& \calT_3^\text{total} = \sum_{m=1}^M  \one\left[\hinerr_m  >\frac{1}{\left(\log \sum_{j=1}^{m} L_j\right)\phi_m} \right]N_m 
\end{align*}
\end{theorem}

\begin{proof}
First of all, by applying Lemma 3 of \cite{kakade2008generalization}, we get that with probability at least $1-\delta$
\begin{align*}
    \forall n \geq 3, 
    \sum_{t=1}^n \one[ y_t \text{is queried}]
    \leq 2\sum_{t=1}^n \E_t\left[\one[ y_t \text{is queried}]\right] + 4\log(4\log(n)/\delta).
\end{align*}
Therefore, to get the actual label complexity w.h.p, it is enough to upper bound the expected one. 

By summing over $\calT_{2,m}$ over all blocks and applying Cauchy-swartz, we have
\begin{align*}
    &\sum_{m=1}^M \sqrt{\ln(2M\log(n)/\delta)\max\{\kappa_m,\frac{N_m}{\ln(2M\log(n)/\delta)}\} N_m}\\
        &\quad + \sum_{m=1}^M \sqrt{\frac{L_m}{\sum_{j=1}^{m-1}L_j}\log(|\calH|n/\delta)\left(\log \sum_{j=1}^{m} L_j\right)L_m\E[\one[ x \in D_m]]\phi_m}\\
        &\quad  + \sum_{m=1}^M \ln(2M\log(n)/\delta) \left(\log \sum_{j=1}^{m} L_j\right)\phi_m \\
    & \leq \sqrt{\ln(2M\log(n)/\delta) \left(\sum_{m=1}^M \max\{\kappa_m,\frac{N_m}{\ln(2M\log(n)/\delta)}\}\right) \sum_{m=1}^M N_m } \\
       & \quad + \sqrt{\log(|\calH|n/\delta)\left(\log \sum_{j=1}^{m} L_j\right) \sum_{L_m} L_m\E[\one[ x \in D_m]] \sum_{m=1}^M  \frac{L_m}{\sum_{j=1}^{m-1}L_j}\phi_m }\\
        &\quad + \sum_{m=1}^M \ln(2M\log(n)/\delta) \left(\log \sum_{j=1}^{m} L_j\right)\phi_m \\
\end{align*}
\end{proof}

\begin{theorem}[Worst-case total label complexity under benign setting (A detailed version))]
Choose $L_m = 2^m$. Suppose there exists an generalized WL error upper bound $\widetilde{\hinerr} \in [0,1]$ that
\begin{align*}
    \sum_{m=1}^M \calT_{1,m} + \calT_{3,m}
    \leq \widetilde{\hinerr} \sum_{m=1}^M N_m ,
\end{align*}
Then we have with probability at least $1-\delta$, 
%
\begin{align*}
    &\sum_{t=1}^n \one[y_t \text{ is queried}] \\
    &\leq  \order\left(( \tilde{\theta} \sqrt{ \berr_M(h^*)n\log(|\calH|n/\delta) +\sqrt{ \berr_M(h^*)n \log(|\calH|/\delta)^3}}
            + \sqrt{\tilde{\theta} \log(|\calH|/\delta)^{\frac{3}{2}}}
            + \log(|\calH|n/\delta)\log(n)^2 \theta^*\right)\\
        & \quad + \order\left( \theta^* \widetilde{\hinerr}~\berr_M(h^*)n + \theta^* \widetilde{\hinerr} \sqrt{ \berr_M(h^*)n \log(|\calH|/\delta)}\right)
\end{align*}
where $\tilde{\theta} = \theta^*\log(n) + \sqrt{\theta^* \left(\sum_{m=1}^M \max\{\kappa(D_m),\frac{N_m}{\ln(2M\log(n)/\delta)}\}\right)\frac{\ln(\log(n)^2/\delta)}{ \log(|\calH|n/\delta)}}$
\end{theorem}
\begin{proof}
    According to \cite{https://doi.org/10.48550/arxiv.1506.08669}, we have
    \begin{align*}
        N_\text{total}
        \leq \sum_{L_m} \E[\one[ x \in D_m]]
        \leq \otil\left( \theta^* \berr_M(h^*)n + \theta^* \sqrt{ \berr_M(h^*)n \log(|\calH|/\delta)} + \log(|\calH|/\delta)\right )
    \end{align*}
    Therefore, by the definition of $\tilde{\theta}$, we have
    \begin{align*}
        \calT_2^\text{total}
        & \lessapprox \sqrt{\ln(\log(n)^2/\delta) \left(\sum_{m=1}^M\max\{\kappa_m,\frac{N_m}{\ln(2M\log(n)/\delta)}\}\right) N_\text{total} } 
        + \sqrt{\log(|\calH|n/\delta)\log(n) \sum_{L_m} L_m\E[\one[ x \in D_m]] \sum_{m=1}^M \phi_m }\\
        &\quad+ \sum_{m=1}^M \ln(\log(n)^2/\delta) \log(n)\phi_m \\
        & \lessapprox  \sqrt{ \theta^* \berr_M(h^*)n + \theta^* \sqrt{ \berr_M(h^*)n \log(|\calH|/\delta)} + \log(|\calH|/\delta)}\\
            & \quad * \left(\sqrt{\ln(\log(n)^2/\delta) \left(\sum_{m=1}^M\max\{\kappa_m,\frac{N_m}{\ln(2M\log(n)/\delta)}\}\right) }
                + \sqrt{\log(|\calH|n/\delta)\log(n)^2 \theta^*} \right) \\
            & \quad + \log(|\calH|n/\delta)\log(n)^2 \theta^* \\
        & \leq  \tilde{\theta} \sqrt{ \berr_M(h^*)n +\sqrt{ \berr_M(h^*)n \log(|\calH|/\delta)}}
            + \sqrt{\tilde{\theta} \log(|\calH|/\delta)}
            + \log(|\calH|n/\delta)\log(n)^2 \theta^*
    \end{align*}
\end{proof}

\subsection{Auxiliary lemma}

\begin{lemma}
\label{lem: upper bound of 1/berr}
When $\calE_\WL$ and $\calE_\text{train}$ holds, for any block m , we have
\begin{align*}
    \frac{1}{\berr(h_m,\Tilde{Z}_{m-1})} \leq \frac{5 + 8\log \left(\sum_{j=1}^{m} L_j\right)+\frac{\eta}{2}}{\berr_m(h^*) + \log \left(\sum_{j=1}^{m} L_j\right) \epsilon_{m}}
\end{align*}
which is equivalent to 
\begin{align*}
    \frac{\E[x \in D_m]}{\berr(h_m,\Tilde{Z}_{m-1})}
    \leq \frac{5 + 8\log \left(\sum_{j=1}^{m} L_j\right)+\frac{\eta}{2}}{\berr_m(h^*) + \log \left(\sum_{j=1}^{m} L_j\right) \epsilon_{m}}\E[x \in D_m]
    \lessapprox \log(\tau_m) \phi_m
\end{align*}
\end{lemma}
\begin{proof}
\begin{align*}
    & \err(h_m,\Tilde{Z}_{m-1})+ \Delta_{m-1} \\
    &\geq \frac{3}{2(5 + 8\log \left(\sum_{j=1}^{m} L_j\right)+\frac{\eta}{2})} \err(h_m,\Tilde{Z}_{m-1})+ \Delta_{m-1}\\
    & =  \frac{1}{2(5 + 8\log \left(\sum_{j=1}^{m} L_j\right)+\frac{\eta}{2})} \left( 3\err(h_m,\Tilde{Z}_{m-1})+ 2(5 + 8\log \left(\sum_{j=1}^{m} L_j\right)+\frac{\eta}{2})\Delta_{m-1}\right)\\
    & \geq \frac{1}{2(5 + 8\log \left(\sum_{j=1}^{m} L_j\right)+\frac{\eta}{2})} \left( 3\err(h_m,\Tilde{Z}_{m-1})+ (5 + 8\log \left(\sum_{j=1}^{m} L_j\right)+\frac{\eta}{2})\Delta_{m-1} + \log \left(\sum_{j=1}^{m} L_j\right) \epsilon_{m-1}\right)\\
    & \geq \frac{\berr_m(h^*) + \log \left(\sum_{j=1}^{m} L_j\right) \epsilon_{m}}{2(5 + 8\log \left(\sum_{j=1}^{m} L_j\right)+\frac{\eta}{2})}
\end{align*}
where the last inequality comes from Theorem~\ref{them: general}. 
\end{proof}

\begin{lemma}
\label{lem: upper bound of P_min}
    For any block m , we have
    \begin{align*}
        P_{m,\min}
        \lessapprox \frac{c_3}{\sqrt{\frac{\left(\sum_{j=1}^{m-1} L_j\right) \berr(h_m,\Tilde{Z}_{m-1})}{n \epsilon_M}}}
    \end{align*}
\end{lemma}
\begin{proof}
First, it is easy to see that, when $\err(f_{m},\tilde{Z}_{m-1}) \geq \log \left(\sum_{j=1}^{m} L_j\right)\epsilon_{m-1}$, we have
\begin{align*}
     \frac{\left(\sum_{j=1}^{m-1} L_j\right) \left(\err(h_m,\Tilde{Z}_{m-1})+ \Delta_{m-1} \right)}{n \epsilon_M}
     \lessapprox \frac{\left(\sum_{j=1}^{m-1} L_j\right) \left(\err(h_m,\Tilde{Z}_{m-1})\right)}{n \epsilon_M}
\end{align*}
On the other hand, when $\err(f_{m},\tilde{Z}_{m-1}) < \left(\log \sum_{j=1}^{m} L_j\right)\epsilon_{m-1}$, we have
\begin{align*}
    \frac{\left(\sum_{j=1}^{m-1} L_j\right) \left(\err(h_m,\Tilde{Z}_{m-1})+ \Delta_{m-1} \right)}{n \epsilon_M}
    \lessapprox \frac{\left(\sum_{j=1}^{m-1} L_j\right) \log \left(\sum_{j=1}^{m} L_j\right) \epsilon_{m-1}}{n \epsilon_M}
     \lessapprox \log \left(\sum_{j=1}^{m} L_j\right).
\end{align*}
Combine these two inequalities we have
\begin{align*}
    \sqrt{\frac{\left(\sum_{j=1}^{m-1} L_j\right) \berr(h_m,\Tilde{Z}_{m-1})}{n \epsilon_M}}
    &: = \sqrt{\frac{\left(\sum_{j=1}^{m-1} L_j\right) \left(\err(h_m,\Tilde{Z}_{m-1})+ \Delta_{m-1} \right)}{n \epsilon_M}}\\
    &\lessapprox \sqrt{\frac{\left(\sum_{j=1}^{m-1} L_j\right) \left(\err(h_m,\Tilde{Z}_{m-1} \right)}{n \epsilon_M}}
    + \log \left(\sum_{j=1}^{m} L_j\right).
\end{align*}
Therefore, by definition of $P_{m, \min}$, we get the target bound.
\end{proof}
\section{Appendix: Practical WL-AC and Experiments}

\subsection{A summary of algorithm modification}
The main difference between the practical and the theoretical version of WL-AC, as stated in Section~\ref{sec: algo (main)}, is the choice of plugable base AL strategy used to calculate the disagreement region as well as its corresponding query probability planning. Other than this, we explain more subtle modifications here, which should not effect the high level picture but help the practical implementation.

\paragraph{A more deterministic block schedule}
In the original templates, only the length of Phase 2 inside each block is fixed as $L_m$ and the length of Phase 1 is a random variable based on the algorithm and data distribution. Such setting is easy for analysis but inconvenient for implementation. Therefore, here we fix total length of each block, including Phase 1 and Phase 2, as $L_m$. (So in Algo~\ref{algo: main-inefficient (supp)} we use ``scheduled training data collection length" and here in Algo~\ref{algo: main (implement))}, we use ``scheduled block length".) Therefore, instead of comparing the newly queried number of samples in Phase 1 with the expected query number in the \textit{fixed} Phase 2, we will compare that with the \textit{remaining} expected query number in Phase 2 while keep the total block length the same. In the following Algo~\ref{algo: hint-eval (implement)}, we use subscript $``+"$ to denote the newly added samples in Phase 1 (WL-EVAL) and use subscript $``-"$ to refer to the remaining samples in Phase 2.

\paragraph{Estimate the sample distribution}
The real $\E[\one[x \in D_m]]$ is not accessible, so we instead assume each data is uniformly distributed and estimate the distribution by the past observed data as shown in Line~\ref{line: emprical distribution}.

\paragraph{A relaxed protocol}
Instead of adhering to the strict streaming setting where each time only the current $x_t$ is observable, we relax our experiment to the "batched`` streaming setting where all the unlabeled context for next block is available at the beginning of the block as shown in Line~\ref{line: batch observation}. Therefore, instead of calculating the expect number of remaining samples within $D_m$ , we can get the \textit{exact} remaining samples in the block $m$, denoted as $\hat{D}_m$. We conjectured that, calculating the expected one using empirical distribution based the past data should yield similar result without this batched assumption. But here we want to rule out the error caused in distribution estimation and only focus on verifying our main strategy. 

\paragraph{Validation data and pseudo loss for neural net}
Several common modifications are required when running neural net. First, instead of using the true loss $\one[y \neq y']$, we use the pseudo loss denoted as $\ell$. Here we choose cross-entropy for multi-class classification. Furthermore, instead of using training loss $\err(h. \tilde{Z}_m)$, we keep track of validation loss, which is usually a better estimates of the true classification error. Note \cite{https://doi.org/10.48550/arxiv.1506.08669} also uses the validation loss when implementing the original AC. The required size of validation set is much smaller than the training set, so its addition to total sample complexity is neglectable.

In addition, this validation set has also been used for early stopping in model training/updating. In the algorithm, we use 
\begin{align*}
    h_{m+1},\err(h_{m+1},Z_\text{val}) = \calO_\text{train}(Z_\text{val},\tilde{Z}_m \cup \Dot{Z}_m)
\end{align*}
to denote this process. That is, the oracle $\calO_\text{train}$ takes in the validation dataset, the training dataset and output the updated model as well as the validation loss.

\paragraph{Less greedy query in Step 3}
In Step 3 of Algo~\ref{algo: hint-eval}, we double the evaluation length in each epoch. In practice, we can increase the length in a milder way to further refine the sample complexity. (Theoretical they have the similar order.) Here we increase a constant number $L_+$ in each iteration as shown in Line~\ref{line: increase L+}.  Note that each iteration only requires an inference step in neural net, so the computational cost is neglectable.

\paragraph{Other parameter choices}
It is well known that the model complexity of neural net is hard to estimate. Therefore, the $|\calH|$-dependent parameters in our algorithms can only be heuristically chosen, indicated by $\approx$ in the algorithm below . We also heuristically choose $\kappa_m =1$. A more comprehensive experiments might be conduct in the future. But the preliminary experiments we conduct here have already shown some positive results. 

\subsection{Algorithm (practical version)}

\label{ sec: experiment-practice (supp)}
\begin{algorithm}[H]

\caption{WL-AC (Practical version)}
\label{algo: main (implement))}
\begin{algorithmic}[1]
\STATE \textbf{Input: } Scheduled block length $L_1,L_2,\ldots$ satisfying $L_{m+1} \leq \sum_{j=1}^m L_j$. Strong labeler $\calO$ and weak labeler $\calW$. Candidate hypothesis set $\calH$ (a neural net model).  \\
a known small validation dataset $Z_\text{val}$ that $|Z_\text{val}| \ll n$ used for hyper-parameter tuning in model training\\
a based AL strategy $\calO_\text{baseAL}$ \\
a training oracle fine-tuned by $Z_\text{val}$ denoted as $\calO_{train}: \calX \times \calY \times Z_\text{val} \to \calH$, and its corresponding pseudo loss $\ell$ 
\STATE \textbf{initialize:} epoch $m=0, \tilde{Z}_{0}:=\emptyset, \Dot{Z}_0: = \emptyset$
\FOR{ $m = 1,2,\ldots, M$}
    \STATE \textcolor{orange}{Phase 1: WL evaluation and query probability assignment}
    \STATE Call \textbf{\textsc{WL-Eval}} in Algo.~\ref{algo: hint-eval (implement)}, and obtain the current \WL~evaluation dataset $\Dot{Z}_m$, the flag $\textsc{Use-WL}$ which decides whether we should transit to NO-WL mode or not and the pessimistically estimated conditional WL accuracy $\Dot{\hinerr}_m$ across $D_m$. 
    \STATE Compute the WL leveraged query probability $\Dot{P}_m = \textbf{\textsc{OP}}(\Dot{\hinerr}_m)$
    \STATE Set the stopping time of Phase 1 as $\Dot{\tau}_m$ and the corresponding length as $\Dot{L}_m$
    \STATE \textcolor{orange}{Phase 2: Train data collection based on planed query probability}
    \STATE Set $S = \emptyset$
    \FOR{$t = \Dot{\tau}_m +1, \ldots,\tau_m + L_m$}
        \IF{$x_t \in D_m$}
            \STATE Draw $Q_t \sim \text{Bernoulli}(\Dot{P}_m)$
            \STATE Update the set of examples:
            \begin{align*}
                S:=
                \begin{cases}
                    S \cup\left\{\left(x_t, y_t, y^\WL_t,1 / \Dot{P}_{m}\left(x_t \right)\right)\right\}, & Q_t=1 \\
                    S \cup\left\{x_t, 1,1,0\right\}, & \text { otherwise. }
                 \end{cases}
            \end{align*}
        \ELSE
            \STATE $S:=  S \cup\left\{\left(x_t, h_m\left(x_t\right),h_m\left(x_t\right), 1\right)\right\} .$ 
        \ENDIF
    \ENDFOR
    \STATE Set $\tilde{Z}_m = \tilde{Z}_{m-1} \cup S$
    \STATE \textcolor{orange}{Phase 3: Model and disagreement region updates using collected train data }
    \STATE Calculate the estimated error for all $h \in \calH$ using shifted double robust estimator and pseudo loss $\ell$
        \begin{align*}
            \err(h,\tilde{Z}_m) =
            \begin{cases}
            \sum_{(x,y,y^\WL,w) \in \tilde{Z}_m} \left(\ell(h(x),y)] - \ell(h(x),y^\WL)\right) w + \ell(h(x),y^\WL) \quad &\textsc{USE-WL}\\
            \sum_{(x,y,y^\WL,w) \in \tilde{Z}_m} \ell(h(x),y)w & \NWL
            \end{cases}
        \end{align*}
    \STATE  Retrain the model using all collected data, get updated model and corresponding validation error $h_{m+1},\err(h_{m+1},Z_\text{val}) = \calO_\text{train}(Z_\text{val},\tilde{Z}_m \cup \Dot{Z}_m)$
    \STATE  Observe the next $L$ context $\{x_t\}_{t=\tau_m}^{\tau_m + L}$ \label{line: batch observation}
    \STATE  Set empirical disagreement region $\hat{D}_{m+1} = \calO_\text{baseAL}(\{x_t\}_{t=\tau_{m-1}+1}^{\tau_{m-1} + L},h_{m+1}, \text{other params})$ 
    \STATE Calculate empirical disagreement probability $\hat{\E}[D_m] = \frac{\calO_\text{baseAL}(\{x_t\}_{t=0}^{\tau_m},h_{m+1}, \text{other params})}{L_m}$ \label{line: emprical distribution}
    \STATE Set the beginning time of the next block as $\tau_{m} = t = \Dot{\tau}_m + L_m$
\ENDFOR
\RETURN $h_M$
\end{algorithmic}
\end{algorithm}

\begin{algorithm}[H]

\small
\caption{\WL~evaluation for block $m$ (WL-EVAL, practical version) }
\label{algo: hint-eval (implement)}
\begin{algorithmic}[1]
\STATE \textbf{Input: } Current disagreement region $\hat{D}_m$, unlabeled samples in this block $\{x_t\}_{t = \tau_m}^{\tau_m + L_m}$ the current optimistic biased estimation of the best error $\err(h_m, Z_\text{val})$
Incremental length in Step3 denoted as $L_+$
\STATE \textcolor{orange}{Step1: Check if the current estimated best error is already good enough}
\STATE Initialize $N_{m,-,\text{label}} = |\hat{D}_m|$, which is the deterministic query number of block $m$ used for training without leveraging WL
\STATE Set $N_{m,+,\text{unlabel}} \approx \left[\frac{1}{\err(h_m, Z_\text{val})} - \Dot{Z}_{m-1}\right]$, which is the deterministic UNLABELLED number we need to newly draw in order to do WL evaluation. 
\STATE Set $N_{m,+,\text{label}} = |\calO_\text{baseAL}(h_{m}, \{x_t\}_{t = \tau_m}^{\tau_m + N_{m,+,\text{unlabel}}})|$, which is the deterministic LABELLED number we need to newly draw in order to do WL evaluation
\IF{$N_{m,-, \text{label}} \lessapprox N_{m,+,\text{label}}$  }
    \RETURN $\textsc{Use-WL} = \text{False}, \Dot{\hinerr}_m = 1, \Dot{Z}_m = \Dot{Z}_{m-1}$
\ENDIF  
\STATE \textcolor{orange}{Step2: Check if the WL performance is worse than $\err(h_m, Z_\text{val})$}
\STATE Set $\kappa_m = 1$, $k=1$
\STATE Draw $N_{m,+,\text{unlabel}}$ number of new unlabeled samples denoted as $\Dot{S}$ and query labels for those inside $D_m$. For each sample at $t$, add them as 
    \begin{align*}
        \Dot{S}:=
        \begin{cases}
            \Dot{S} \cup\left\{x_t, y_t, y^\WL_t\right\}, & Q_t=1 \\
            \Dot{S} \cup\left\{x_t, 1,1\right\}, & \text { otherwise. }
         \end{cases}
    \end{align*}
    and update $\ddot{Z}_{m,k} = \Dot{Z}_{m-1} \cup \Dot{S}$
\STATE Calculate the empirical mean $\wh{\hinerr}_{m,k} = \frac{1}{|\ddot{Z}_{m,k}|} \sum_{(x,y,y^\WL) \in \ddot{Z}_{m,k}} \one[y^\WL \neq y \wedge x \in D_m]$ 
\STATE Calculate the optimistic estimation $\overline{\hinerr}_{m,k} =  \min\{(\wh{\hinerr}_k + \frac{1}{|\ddot{Z}_{m,k}|}), \hat{\E}[x \in D_m] \}$
\IF{$\overline{\hinerr}_k \gtrapprox \err(h_m, Z_\text{val})$}
    \RETURN $\textsc{Use-WL} = \text{False}, \Dot{\hinerr}_m = 1,  \Dot{Z}_m = \ddot{Z}_{m,k}$
\ENDIF
\STATE \textcolor{orange}{Step3: Get a more precise estimation on WL performance}
\WHILE{ $N_{m,-, \text{label}}\gtrapprox N_{m,+,\text{label}}$}
    \STATE $k \leftarrow k+1$
    \STATE Draw $L_+$ new unlabeled samples \label{line: increase L+}
   and query labels for those inside $D_m$. For each sample at $t$, add them as 
    \begin{align*}
        \Dot{S}:=
        \begin{cases}
            \Dot{S} \cup\left\{x_t, y_t, y^\WL_t\right\}, & Q_t=1 \\
            \Dot{S} \cup\left\{x_t, 1,1\right\}, & \text { otherwise. }
         \end{cases}
    \end{align*}
and update $\ddot{Z}_{m,k} = \ddot{Z}_{m,k-1} \cup \Dot{S}$
\STATE Update the total query at the current block $N_{m,+,\text{label}}$
    \STATE Calculate the the optimistic estimation $\overline{\hinerr}_k$ as before
    \STATE Update the remaining number of samples in this block that are planed to be queried, denoted as $N_{m,-, \text{label}}$
\ENDWHILE
\RETURN $\textsc{Use-WL} = \text{True}, \Dot{\hinerr}_m = \min\left\{\frac{ \kappa_m \overline{\hinerr}_{m,k}}{\E[\one[x \in D_m]]},1\right\},  \Dot{Z}_m = \ddot{Z}_{m,k}$ 
\end{algorithmic}
\end{algorithm}

\begin{mdframed}
\label{algo: op (implement)}
\small
{\centering \textbf{Optimization Problem} (OP,practical version) to compute $P_m$ \\}
\textbf{Input: } Estimated hints performance upper bound $\Dot{\hinerr}_m$
\begin{align*}
    P_m = \max\{ \Dot{\hinerr}_m, P_{m,\min}\}
\end{align*}
\end{mdframed}

\subsection{Performance regarding to the number of observed unlabeled samples}
\label{ sec: experiment unlabel results (supp)}

\begin{figure}[t]
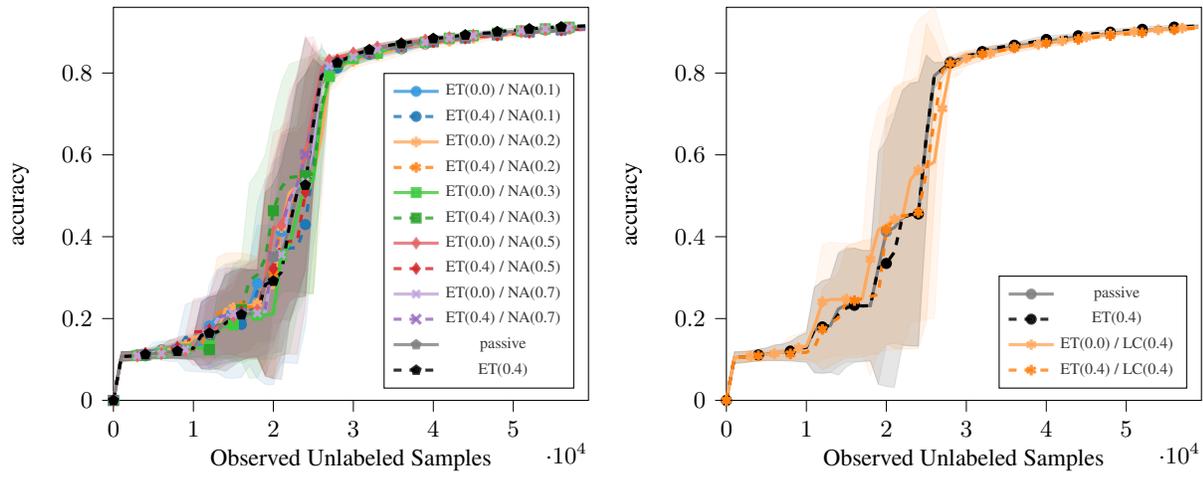

    \begin{subfigure}[b]{0.47\textwidth}
        \centering
        \input{plots/noisy_new_unlabeled.tex}
        \label{fig: noisy_query}
    \end{subfigure}
    \begin{subfigure}[b]{0.47\textwidth}
        \centering
        \input{plots/informative_new_unlabeled.tex}
        \label{fig: informative_queryl}
    \end{subfigure}
    \caption{\ding{182} - Performance comparison under uniform noisy (left) and informative classifier (right) weak labeler, regarding to the number of unlabeled samples.}
    \label{fig:synt_unlabeled}
\end{figure}

\begin{figure}[t]
    \centering
    \input{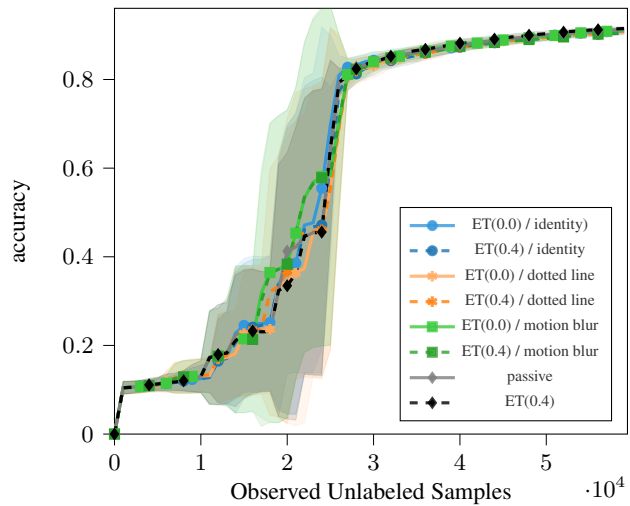}
    \caption{\ding{183} - Performance comparison under pre-trained classifier weak labeler, regarding to the number of unlabeled samples.}
    \label{fig: pretrain_queryl_unlabeled}
\end{figure}


\end{document}